\documentclass{article}
\pdfoutput=1

\usepackage[final]{neurips_2021} 

\usepackage[utf8]{inputenc} 
\usepackage[T1]{fontenc}    
\usepackage{hyperref}       
\usepackage{url}            
\usepackage{booktabs}       
\usepackage{amsfonts}       
\usepackage{nicefrac}       
\usepackage{microtype}      
\usepackage{lipsum}		
\usepackage{graphicx}
\usepackage{natbib}
\usepackage{doi}
\usepackage{bbm}

\usepackage[utf8]{inputenc} 
\usepackage[T1]{fontenc}    
\usepackage{hyperref}       
\usepackage{url}            
\usepackage{booktabs}       
\usepackage{amsfonts}       
\usepackage{nicefrac}       
\usepackage{microtype}      
\usepackage{lipsum}

\usepackage[american]{babel}
\usepackage{filecontents}
\usepackage{microtype}
\usepackage{subfigure}
\usepackage{booktabs} 
\usepackage[normalem]{ulem}
\usepackage{graphicx}
\usepackage{caption}

\usepackage{hyperref}

\usepackage{verbatim}
\usepackage{amsmath}
\allowdisplaybreaks
\usepackage{amsthm}
\usepackage{microtype}
\usepackage{color}
\usepackage{eucal}
\usepackage{epsfig,amsmath,amssymb,amsfonts,amstext,amsthm,mathrsfs}
\usepackage{latexsym,graphics,epsf,epsfig,psfrag}
\usepackage{dsfont,color,epstopdf,fixmath}
\usepackage{enumitem}
\usepackage{algorithm,algorithmic,latexsym}

\newtheorem{assumption}{\textbf{Assumption}}

\newtheorem{Lemma}{\textbf{Lemma}}
\newtheorem{theorem}{\textbf{Theorem}}

\newtheorem{remark}{\textbf{Remark}}

\usepackage[acronym]{glossaries}

\newacronym{T1}{T1}{\theta_1}
\usepackage{color}

\newcommand{\mcs}{\mathcal{S}}
\newcommand{\mca}{\mathcal{A}}
\newcommand{\nn}{\nonumber}
\newcommand{\mE}{\mathbb{E}}
\newcommand{\om}{\omega}

\newcommand{\tb}{\tau_{\beta}}
\newcommand{\tmix}{t_{\text{mix}}}

\usepackage{natbib} 
\usepackage{mathtools} 
\usepackage{siunitx} 
\usepackage{booktabs} 
\usepackage{tikz} 
\usepackage{cleveref}

\usepackage{mathtools} 
\usepackage{siunitx} 
\usepackage{cleveref}
\usepackage[toc,page,header]{appendix}
\usepackage{minitoc}

\title{Online Robust Reinforcement Learning with Model Uncertainty}

\author{%
  Yue Wang  \\
  University at Buffalo\\
  Buffalo, NY 14228 \\
  \texttt{ywang294@buffalo.edu} \\
  \And
  Shaofeng Zou \\
  University at Buffalo\\
  Buffalo, NY 14228 \\
  \texttt{szou3@buffalo.edu} \\
}

\begin{document}
\maketitle
\begin{abstract}
     Robust reinforcement learning (RL) is to find a policy that optimizes the worst-case performance over an uncertainty set of MDPs. In this paper, we focus on \textit{model-free} robust RL, where the uncertainty set is defined to be centering at a misspecified MDP that generates a single sample trajectory sequentially, and is assumed to be \textit{unknown}. We develop a sample-based approach to estimate the unknown uncertainty set, and design robust Q-learning algorithm (tabular case) and robust TDC algorithm (function approximation setting), which can be implemented in an online and incremental fashion.  For the robust Q-learning algorithm, we prove that it converges to the optimal robust Q function, and for the robust TDC algorithm, we prove that it converges asymptotically to some stationary points. Unlike the results in \citep{roy2017reinforcement}, our algorithms do not need any additional conditions on the discount factor to guarantee the convergence. We further characterize the finite-time error bounds of the two algorithms, and show that both the robust Q-learning and robust TDC algorithms converge as fast as their  vanilla counterparts (within a constant factor). Our numerical experiments further demonstrate the robustness of our algorithms. Our approach can be readily extended to robustify many other algorithms, e.g., TD, SARSA, and other GTD algorithms. 
\end{abstract}
\section{Introduction}

Existing studies on {Markov decision process} (MDP) and {reinforcement learning} (RL) \citep{sutton2018reinforcement} mostly rely on the crucial assumption that the environment on which a learned policy will be deployed is the same one that was used to generate the policy, which is often violated in practice -- e.g., the simulator may be different from the true environment, and the MDP may evolve over time. Due to such model deviation, the actual performance of the learned policy can significantly degrade. 
To address this problem, the framework of robust MDP was formulated in \citep{bagnell2001solving,nilim2004robustness,iyengar2005robust}, where the transition kernel of the MDP is not fixed and lies in an uncertainty set, and the goal is to learn a policy that performs well under the worst-case MDP in the uncertainty set. In \citep{bagnell2001solving,nilim2004robustness,iyengar2005robust}, it was assumed that the uncertainty set is known beforehand, i.e., model-based approach, and dynamic programming can be used to find the optimal robust policy. 


The model-based approach, however, requires a model of the uncertainty set known beforehand, and needs a large memory to store the model when the state and action spaces are large, which make it less applicable for many practical scenarios. This motivates the study in this paper, \textit{model-free} robust RL with model uncertainty,  which is to learn a robust policy using a single sample trajectory from a misspecified MDP, e.g., a simulator and a similar environment in which samples are easier to collect than in the target environment where the policy is going to be deployed. 
The major challenge lies in that the transition kernel of the misspecified MDP is not given beforehand, and thus, the uncertainty set and the optimal robust policy need to be learned simultaneously using sequentially observed data from the misspecified MDP. Moreover, robust RL learns the value function of the worst-case MDP in the uncertainty set which is different from the misspecified MDP that generates samples. This is similar to the off-policy learning, which we refer to as the "off-transition-kernel" setting. Therefore, the learning may be unstable and could diverge especially when function approximation is used \citep{baird1995residual}.

In this paper, we develop a model-free approach for robust RL with model uncertainty. Our major contributions in this paper are summarized as follows.
\begin{itemize}[leftmargin=*]
    \item Motivated by empirical studies of adversarial training in RL \citep{huang2017adversarial,kos2017delving,lin2017tactics,pattanaik2018robust,mandlekar2017adversarially} and the $R$-contamination model in robust detection (called $\epsilon$-contamination model in \citep{hub65}), we design the uncertainty set using the $R$-contamination model (see \eqref{eq:uncertainset} for the details). We then develop an approach to estimate the unknown uncertainty set using only the current sample, which does not incur any additional memory cost. Unlike the approach in \citep{roy2017reinforcement}, where  the uncertainty set is relaxed to one not depending on the misspecified MDP that generates samples so that an online algorithm can be constructed, our approach does not need to relax the uncertainty set. 
    \item We develop a robust Q-learning algorithm for the tabular case,  which can be implemented in an online and incremental fashion, and has the same memory cost as the vanilla Q-learning algorithm. We show that our robust Q-learning algorithm converges asymptotically, and further characterize its finite-time error bound. Unlike the results in \citep{roy2017reinforcement} where a stringent condition on the discount factor (which is due to the relaxation of the uncertainty set, and prevents the use of a discount factor close to 1 in practice) is needed to guarantee the convergence, our algorithm converges without the need of such condition. Furthermore, our robust Q-learning algorithm converges as fast as the vanilla Q-learning algorithm \citep{li2020sample} (within a constant factor), while being robust to model uncertainty.
    \item We generalize our approach to the case with function approximation (for large state/action space). We investigate the robust policy evaluation problem, i.e., evaluate a given policy under the worst-case MDP in the uncertainty set. As mentioned before, the robust RL problem is essentially "off-transition-kernel", and therefore  non-robust methods with function approximation may diverge  \citep{baird1995residual} (also see our experiments). We develop a novel extension of the gradient TD (GTD) method \citep{maei2010toward,maei2011gradient,sutton2008convergent} to robust RL. Our approach introduces a novel smoothed robust Bellman operator to construct the smoothed mean-squared projected robust Bellman error (MSPRBE). Using our uncertainty set design and online sample-based estimation, we develop a two time-scale robust TDC algorithm. We further characterize its convergence and finite-time error bound.
    \item We conduct numerical experiments to validate the robustness of our approach. In our experiments, our robust Q-learning algorithm achieves a much higher reward than the vanilla Q-learning algorithm when being trained on a misspecified MDP; and our robust TDC algorithm converges much faster than the vanilla TDC algorithm, and the vanilla TDC algorithm may even diverge. 
\end{itemize}

\subsection{Related Work}

\textbf{Model-Based Robust MDP.}
The framework of robust MDP was investigated in \citep{iyengar2005robust,nilim2004robustness,bagnell2001solving,satia1973markovian,wiesemann2013robust}, where the transition kernel is assumed to be in some uncertainty set, and the problem can be solved by dynamic programming. This approach was further extended to the case with function approximation in \citep{tamar2014scaling}. However, these studies are model-based, which assume beforehand knowledge of the uncertainty set. In this paper, we investigate the model-free setting, where the uncertainty set is a set of MDPs centered around some unknown Markov transition kernel from which a single sample trajectory can be sequentially observed. 

\textbf{Adversarial Robust RL.} It was shown in \citep{iyengar2005robust} that the robust MDP problem is  equivalent to a zero-sum game between the agent and the nature. 
Motivated by this fact, the adversarial training approach, where an adversary perturbs the state transition, was studied in \citep{vinitsky2020robust,pinto2017robust,abdullah2019wasserstein,hou2020robust,rajeswaran2017epopt,atkeson2003nonparametric,morimoto2005robust}. This method relies on a simulator, where the state transition can be modified in an arbitrary way. Another approach is to modify the current state through adversarial samples, which is more heuristic, e.g., \citep{huang2017adversarial,kos2017delving,lin2017tactics,pattanaik2018robust,mandlekar2017adversarially}. Despite the empirical success of these approaches, theoretical performance guarantees, e.g., convergence to the optimal robust policy and convergence rate, are yet to be established. The main difference lies in that during the training, our approach does not need to manipulate the state transition of the MDP. More importantly, we develop the asymptotic convergence to the optimal robust policy and further characterize the finite-time error bound. 
In \citep{lim2013reinforcement},  the scenario where  some unknown parts of the state space can have arbitrary transitions while other parts are purely stochastic was studied. Adaptive algorithm to adversarial behavior was designed, and its regret bound is shown to be similar to the purely stochastic case. In \citep{zhang2020stability}, the robust adversarial RL problem for the special  linear quadratic case was investigated. 

\textbf{Model-free Robust RL.} In \citep{roy2017reinforcement,badrinath2021robust} model-free RL with model uncertainty was studied, where in order to construct an algorithm that can be implemented in an online and incremental fashion, the uncertainty set was firstly relaxed by dropping the dependency on the misspecified MDP that generates the samples (centroid of the uncertainty set). Such a relaxation is pessimistic since the relaxed uncertainty set is not centered at the misspecified MDP anymore (which is usually similar to the target MDP), making the robustness to the relaxed uncertainty set not well-justified. Such a relaxation will further incur a stringent condition on the discounted factor to guarantee the convergence, which prevents the use of a discount factor close to 1 in practice. Moreover, only asymptotic convergence was established in \citep{roy2017reinforcement}. 
In this paper, we do not relax the uncertainty set, and instead propose an online approach to estimate it. Our algorithms converge without the need of the condition on the discount factor. We also provides finite-time error bounds for our algorithms. The multi-agent RL robust to reward uncertainty was investigated in \citep{zhang2020robust}, where the reward uncertainty set is known, but the transition kernel is fixed.

\textbf{Finite-time Error Bound for RL Algorithms.} 
For the tubular case, Q-learning has been studied intensively, e.g., in \citep{even2003learning,beck2012error,qu2020finite,li2020sample,wainwright2019variance,li2021q}. TD with  function approximation were studied in \citep{Dalal2018a,bhandari2018finite,srikant2019,cai2019neural,sun2020finite}.  Q-learning and SARSA with linear function approximation
were investigated in \citep{zou2019finite,chen2019performance}. The finite-time error bounds for
the gradient TD algorithms \citep{maei2010toward,sutton2009fast,maei2010toward}
were further developed recently in \citep{dalal2018finite,liu2015finite,gupta2019finite,xu2019two,dalal2020tale,kaledin2020finite,ma2020variance,wang2020finite,ma2021greedygq,doan2021finite}. 
There are also finite-time error bounds on the policy gradient methods and actor critic methods, e.g., \citep{Wang2020Neural,yang2019provably,kumar2019sample,qiu2019finite,wu2020finite,cen2020fast,bhandari2019global,agarwal2021theory,mei2020global}.
We note that these studies are for the non-robust RL algorithms, and in this paper, we design robust RL algorithms, and characterize their finite-time error bounds.  

\section{Preliminaries}
\textbf{Markov Decision Process.}
An MDP  can be characterized by a tuple  $(\mathcal{S},\mathcal{A},  \mathsf P, c, \gamma)$, where $\mcs$ and $\mca$ are the state and action spaces, $\mathsf P=\left\{p^a_s \in \Delta_{|\mcs|}, a\in\mca, s\in\mcs\right\}$ is the transition kernel\footnote{$\Delta_n$ denotes the $(n-1)$-dimensional probability simplex: $\{(p_1,...,p_n)| 0\leq p_i\leq 1, \sum^n_{i=1}p_i=1 \}$.}, $c$ is the cost function, and $\gamma\in[0,1)$ is the discount factor. Specifically, $p^a_s$ denotes the distribution of the next state if taking action $a$ at state $s$. Let $p^a_s=\{p^a_{s,s'}\}_{s'\in\mathcal S}$, where $p^a_{s,s'}$ denotes the probability that the environment transits to state $s'$ if taking action $a$ at state $s$. The cost of taking action $a$ at state $s$ is given by $c(s,a)$. A stationary policy $\pi$ is a mapping from $\mcs$ to a distribution over $\mca$. 
At each time $t$, an agent takes an action $A_t\in\mca$ at state $S_t\in\mcs$. The environment then transits to the next state $S_{t+1}$ with probability $ p^{A_t}_{S_t,S_{t+1}}$, and the agent receives cost given by $c(S_t,A_t)$. 
The value function of a policy $\pi$ starting from any initial state $s\in\mcs$ is defined as the expected accumulated discounted cost by following $\pi$: 
$\mathbb{E}\left[\sum_{t=0}^{\infty}\gamma^t   c(S_t,A_t )|S_0=s,\pi\right]$, and the goal is to find the policy $\pi$ that minimizes the above value function for any initial state $s\in\mcs$.





\textbf{Robust Markov Decision Process.}
In the robust case, the transition kernel is not fixed and lies in some uncertainty set.
Denote the transition kernel at time $t$ by $\mathsf P_t$, and let $\kappa=(\mathsf P_0,\mathsf P_1,...)$, where $\mathsf P_t\in\mathbf P, \forall t\geq 0$, and $\mathbf P$ is the uncertainty set of the transition kernel. 
The sequence $\kappa$ can be viewed as the policy of the nature, and is adversarially chosen by the nature \citep{bagnell2001solving,nilim2004robustness,iyengar2005robust}. Define the robust value function of a policy $\pi$ as the worst-case expected accumulated discounted cost following a fixed policy $\pi$ over all transition kernels in the uncertainty set:
\begin{align}\label{eq:value}
    V^{\pi}(s)=\max_{\kappa} \mathbb{E}_{ \kappa}\left[\sum_{t=0}^{\infty}\gamma^t   c(S_t,A_t )|S_0=s,\pi\right],
\end{align}
where $\mE_{\kappa}$ denotes the expectation when the state transits according to $\kappa$. 
%
%
%
%
%
Similarly, define the robust action-value function for a policy $\pi$:
$
    Q^{\pi}(s,a)=\max_{\kappa }\mathbb{E}_{ \kappa}\left[\sum_{t=0}^{\infty}\gamma^t   c(S_t,A_t )|S_0=s,A_0=a,\pi\right].
$
The goal of robust RL is to find the optimal robust policy $\pi^*$  that minimizes the worst-case accumulated discounted cost:
\begin{flalign}
\pi^*=\arg\min_{\pi} V^{\pi}(s), \forall s\in\mcs.
\end{flalign}
We also denote $V^{\pi^*}$ and $Q^{\pi^*}$ by $V^*$ and $Q^*$, respectively,
%
and $V^*(s)=\min_{a\in \mca} Q^*(s,a)$.

Note that a transition kernel is a collection of conditional distributions. Therefore, the uncertainty set $\mathbf P$ of the transition kernel can be equivalently written as a collection of $ \mathcal P_s^a$ for all ${s\in\mcs, a\in\mca}$, where $\mathcal P_s^a$ is a set of conditional distributions $p_s^a$ over the state space $\mcs$. Denote by $\sigma_{\mathcal{P}}(v)\triangleq \max_{p\in\mathcal{P}}(p^\top v)$ the support function of vector $v$ over a set of probability distributions $\mathcal{P}$.
For robust MDP, the following robust analogue of the Bellman recursion was provided in \citep{nilim2004robustness,iyengar2005robust}.
\begin{theorem} \citep{nilim2004robustness}\label{thm:dualtiy}
The following perfect duality condition holds for all $s\in\mcs$:
\begin{align}
    \min_{\pi}\max_{\kappa}\mE_{\kappa}\left[ \sum^{\infty}_{t=0} \gamma^t c(S_t,A_t)\big|\pi, S_0=s\right]=\max_{\kappa}\min_{\pi}\mE_{\kappa}\left[ \sum^{\infty}_{t=0} \gamma^t c(S_t,A_t)\big|\pi, S_0=s\right].
\end{align}
The optimal robust  value function $V^*$ satisfies
$
    V^*(s)=\min_{a\in\mca} (c(s,a)+\gamma \sigma_{\mathcal{P}^a_s}(V^*)),
$
and the optimal robust action-value function $Q^*$ satisfies
$
    Q^*(s,a)=c(s,a)+\gamma \sigma_{\mathcal{P}^a_s}(V^*).
$
\end{theorem}

Define the robust Bellman operator $\mathbf T$ by $\mathbf T Q(s,a)=c(s,a)+\gamma \sigma_{\mathcal{P}^a_s}(\min_{a\in\mca}Q(s,a))$. It was shown in \citep{nilim2004robustness,iyengar2005robust} that $\mathbf T$ is a contraction and its fixed point is the optimal robust $Q^*$. 
When the uncertainty set is known, so that $\sigma_{\mathcal{P}^a_s}$ can be  computed exactly, $V^*$ and $Q^*$ can be solved by dynamic programming \citep{iyengar2005robust,nilim2004robustness}.

 \section{R-Contamination Model For Uncertainty Set Construction}
In this section, we construct the uncertainty set using the $R$-contamination model. 

Let $\mathsf P=\{p_s^a,s\in\mcs, a\in\mca\}$ be the centroid of the uncertainty set, i.e., the transition kernel that generates the sample trajectory, and $\mathsf P$ is \textit{unknown}. 
For example, $\mathsf P$ can be  the simulator at hand, which may not be exactly accurate; and $\mathsf P$ can be the transition kernel of environment 1, from which we can take samples to learn a policy that will be deployed in a similar environment 2. 
The goal is to learn a policy using samples from $\mathsf P$ that performs well when applied to a perturbed MDP from $\mathsf P$.

Motivated by empirical studies of adversarial training in RL \citep{huang2017adversarial,kos2017delving,lin2017tactics,pattanaik2018robust,mandlekar2017adversarially} and the $R$-contamination model in robust detection \citep{hub65}, we use the $R$-contamination model to define the uncertainty set:
\begin{flalign}\label{eq:uncertainset}
\mathcal{P}^a_s=\left\{(1-R)p^a_s+Rq|q\in\Delta_{|\mcs|} \right\},  s\in\mcs, a\in\mca, \text{ for some } 0\leq R\leq 1.
\end{flalign}
 Here, $p^a_s$ is the centroid of the uncertainty set $\mathcal{P}^a_s$ at $(s,a)$, which is \textit{unknown}, and $R$ is the design parameter of the uncertainty set, which measures the size of the uncertainty set, and is assumed to be known in the algorithm. We then let $\mathbf P=\bigotimes_{s\in\mcs,a\in\mca}\mathcal P_s^a$. 

\begin{remark}
$R$-contamination model is closely related to other uncertainty set models like total variation and KL-divergence. It can be shown that $R$-contamination set certered at $p$ is a subset of total variation ball : $\left\{(1-R)p+Rq|q\in\Delta_{|\mcs|} \right\} \subset \left\{q\in\Delta_{|\mcs|}| d_{\text{TV}}(p,q)\leq R  \right\}$. Hence the total variation uncertainty set is less conservative than our $R$-contamination uncertainty set. KL-divergence moreover can be related to total variation using Pinsker’s inequality, i.e., $d_{\text{TV}}(p,q)\leq \sqrt{\frac{1}{2}d_{\text{KL}}(p,q)}$. 
\end{remark}

\section{Tabular Case:  Robust Q-Learning}\label{sec:tabular}
In this section, we focus on the tabular case with finite state and action spaces. We focus on the asynchronous setting where a single sample trajectory is available with Markovian noise.
%
We will develop an efficient approach to estimate the unknown uncertainty set $\mathbf P$, and further the support function $\sigma_{\mathcal P_s^a}(\cdot)$,
%
%
%
%
and  then design our robust Q-learning algorithm.



 
We propose an efficient and data-driven approach to estimate the unknown $p^a_{s}$ and thus the unknown uncertainty set $\mathcal P_s^a$ for any $s\in\mcs$ and $a\in\mca$. Specifically, 
denote the sample at $t$-th time step  by $O_t=(s_t,a_t,s_{t+1})$. We then use $O_t$ to obtain the maximum likelihood estimate (MLE) $\hat{{p}}_t\triangleq \mathbbm{1}_{s_{t+1}}$ of the transition kernel $p_{s_t}^{a_t}$, where $\mathbbm{1}_{s_{t+1}}$ is a probability distribution taking probability $1$ at $s_{t+1}$ and $0$ at other states. This is an unbiased estimate of the transition kernel $p^{a_t}_{s_t}$ conditioning on $S_t=s_t$ and $A_t=a_t$. 
We then design a sample-based estimate  
$
\hat{\mathcal{P}}_t\triangleq \left\{(1-R)\hat p_t+Rq|q\in\Delta_{|\mcs|} \right\}
$
of the uncertainty set $\mathcal P_{s_t}^{a_t}$.
%
%
Using the sample-based uncertainty set $\hat{\mathcal{P}}_t$, we construct the following robust Q-learning algorithm in \Cref{alg:1}.
\begin{algorithm} [!htb]
\caption{Robust Q-Learning}
\label{alg:1}
\mbox{\textbf{Initialization}: $T$, $Q_0(s,a)$ for all $(s,a)\in\mcs\times\mca$}, behavior policy $\pi_b$, $s_0$,  step size $\alpha_t$

\begin{algorithmic}[1] 
\FOR {$t=0,1,2,...,T-1$}
		\STATE {Choose $a_t$ according to $\pi_b(\cdot|s_t)$}
		\STATE Observe $s_{t+1}$ and $c_{t}$
		\STATE $V_t(s)\leftarrow \min_{a\in\mathcal{A}} Q_t(s,a)$, $\forall s\in\mcs$
		\STATE $Q_{t+1}(s_t,a_t) \leftarrow (1-\alpha_t)Q_t(s_t,a_t)+\alpha_t(c_t+\gamma \sigma_{\hat{\mathcal{P}}_t}(V_t))$
		\STATE $Q_{t+1}(s,a)\leftarrow Q_t(s,a)$ for $(s,a)\neq (s_t,a_t)$
		\ENDFOR
\end{algorithmic}
\textbf{Output}: $Q_T$
\end{algorithm}
For any $t$,  $\sigma_{\hat{\mathcal{P}}_t}(V_t)$ can be easily computed:
$
    \sigma_{\hat{\mathcal{P}}_t}(V_t)=R \max_{s\in\mcs}V_{t}(s) +(1-R) V_{t}(s_{t+1}). 
$
Hence the update in Algorithm \ref{alg:1} (line 5) can be  written as 
\begin{align}\label{eq:Qupdate}
    Q_{t+1}(s_t,a_t) \leftarrow (1-\alpha_t)Q_t(s_t,a_t)+\alpha_t(c_t+\gamma R \max_{s\in\mcs}V_{t}(s) +\gamma (1-R) V_{t}(s_{t+1}) ).
\end{align} 

Compared to the model-based approach, our approach is model-free. It does not require the prior knowledge of the uncertainty set, i.e., the knowledge of $p_s^a, \forall s\in\mcs, a\in\mca$. Furthermore, the memory requirement of our algorithm is $|\mcs|\times|\mca|$ (used to store the Q-table), and unlike the model-based approach it does not need a table of size $|\mcs|^2|\mca|$ to store $p_s^a,\forall s\in\mcs, a\in\mca$, which could be problematic if the state space is large. Moreover, our algorithm does not involve a relaxation of the uncertainty set like the one in \citep{roy2017reinforcement}, which will incur a stringent condition on the discount factor to guarantee the convergence. As will be shown below, the convergence of our \Cref{alg:1} does not require any condition on the discount factor. 

We show in the following theorem that the robust Q-learning algorithm converges asymptotically to the optimal robust action-value function $Q^*$.
\begin{theorem}(Asymptotic Convergence)\label{thm:Qasyconv}
If step sizes $\alpha_t$ satisfy that $\sum^{\infty}_{t=0} \alpha_t =\infty$ and $\sum^{\infty}_{t=0} \alpha_t^2 <\infty$, then  $Q_t\to Q^*$ as $t\rightarrow\infty$ with probability 1. 
\end{theorem}

To further establish the finite-time error bound for our robust Q-learning algorithm in \Cref{alg:1}, we make the following assumption that is commonly used in the analysis of vanilla Q-learning.
%
\begin{assumption}\label{ass1}
The Markov chain induced by the behavior policy $\pi_b$ and the transition kernel $p_s^a,\forall s\in\mcs, a\in\mca$ is uniformly ergodic.   
\end{assumption}
Let $\mu_{\pi_b}$ denote the stationary distribution over $\mcs\times\mca$ induced by $\pi_b$ and $p_s^a,\forall s\in\mcs, a\in\mca$. We then further define $\mu_{\min}=\min_{(s,a)\in\mcs\times\mca} \mu_{\pi_b}(s,a)$. This quantity characterizes how many samples are needed to visit every state-action pair sufficiently often. 
Define the following mixing time of the induced Markov chain:
$
\tmix=\min \left\{ t: \max_{s\in\mcs} d_{\text{TV}}(\mu_\pi,P(s_t=\cdot|s_0=s) ) \leq \frac{1}{4}\right\},
$
where $d_{\text{TV}}$ is the total variation distance.

The following theorem establishes the finite-time error bound of our robust Q-learning algorithm. 

\begin{theorem}(Finite-Time Error Bound)\label{thm:Qbound}
There exist some positive constants $c_0$ and $c_1$ such that for any $\delta<1$, any $\epsilon<\frac{1}{1-\gamma}$,  any $T$ satisfying 
\begin{flalign}
T\geq c_0\left(\frac{1}{\mu_{\text{min}}(1-\gamma)^5\epsilon^2}+\frac{\tmix}{\mu_{\text{min}}(1-\gamma)} \right)\log \left(\frac{T|\mcs||\mca|}{\delta}\right)\log\left( \frac{1}{\epsilon(1-\gamma)^2}\right),
\end{flalign}
and step size 
$
\alpha_t=\frac{c_1}{\log \left(\frac{T|\mcs||\mca|}{\delta}\right)}\min \left(\frac{1}{\tmix},\frac{\epsilon^2(1-\gamma)^4}{\gamma^2} \right), \forall t\geq 0
$
we have with probability at least $1-6\delta$,
$
    \|Q_T-Q^*\|_{\infty}\leq 3\epsilon.
$
\end{theorem}
From the theorem, we can see that to guarantee an $\epsilon$-accurate  estimate, a sample size 
$\tilde {\mathcal O}(\frac{1}{\mu_{\text{min}}(1-\gamma)^5\epsilon^2}+\frac{\tmix}{\mu_{\text{min}}(1-\gamma)} )$ (up to some logarithmic terms)  is needed. This complexity  matches with the one for the vanilla Q-learning in \citep{li2020sample} (within a constant factor), while our algorithm also guarantees robustness to MDP model uncertainty. 
Our algorithm design and analysis can be readily extended to robustify TD and SARSA. 
The variance-reduction technique \citep{wainwright2019variance} can also be combined with our robust Q-learning algorithm to further improve the dependency  on $(1-\gamma)$.




\section{Function Approximation: Robust TDC}
In this section, we investigate the case where the state and action spaces can be large or even continuous. A popular approach is to approximate the value function using a parameterized function, e.g., linear function and neural network. In this section, we focus on the case with linear function approximation to illustrate the main idea of designing robust RL algorithms. Our approach can be extended to non-linear (smooth) function approximation using techniques in, e.g., 
\citep{cai2019neural,bhatnagar2009convergent,wai2019variance,wang2021finite}.

We focus on the problem of robust policy evaluation, i.e., estimate the robust value function $V^\pi$ defined in \eqref{eq:value} for a given policy $\pi$ under the worst-case MDP transition kernel in the uncertainty set. 
%
%
%
%
Note that for robust RL with model uncertainty, any policy evaluation problem can be viewed as "off-transition-kernel", as it is to evaluate the value function under the worst-case MDP using samples from a different MDP. 
%
%
Since the TD algorithm with function approximation may diverge under off-policy training \citep{baird1995residual} and importance sampling cannot be applied here due to unknown transition kernel, in this paper we generalize the GTD method \citep{maei2010toward,maei2011gradient} to the robust setting.

Let $\left\{ \phi^{(i)}: \mcs\rightarrow \mathbb R,\, i=1,\ldots,N \right\}$ be a set of $N$ fixed base functions, where $N \ll |\mcs||\mca|$. In particular, we approximate the robust value function using a linear combination of $\phi^{(i)}$'s:  
$
V_\theta(s)=\sum_{i=1}^N \theta^i\phi^{(i)}_{s}=\phi_{s}^\top \theta,
$
where $\theta \in \mathbb{R}^N$ is the weight vector.


Define the following robust Bellman operator for a given policy $\pi$:
\begin{align}\label{eq:bellman op}
   \mathbf T_{\pi}V(s)&\triangleq \mE_{A\sim\pi(\cdot|s)}[c(s,A)+\gamma \sigma_{\mathcal{P}^{A}_{s}}(V)]\nn\\
     &=\mE_{A\sim\pi(\cdot|s)}\left[c(s,A)+\gamma(1-R)\sum_{s'\in\mcs} p^{A}_{s,s'} V(s')+\gamma R \max_{s'\in\mcs}V(s')\right].
\end{align}
We then define the mean squared projected robust Bellman error (MSPRBE) as 
\begin{align}
    \text{MSPRBE}(\theta)=\left\| \mathbf\Pi \mathbf T_{\pi} V_{\theta}-V_{\theta}\right\|^2_{\mu_\pi},
\end{align}
where $\|v\|^2_{\mu_\pi}=\int v^2(s) \mu_\pi(ds)$, $\mu_\pi$ is the stationary distribution induced by $\pi$, and $\mathbf\Pi$ is a projection onto the linear function space w.r.t. $\|\cdot\|_{\mu_\pi}$. 
We will develop a two time-scale gradient-based approach to minimize the MSPRBE. However, it can be seen that $\max_s V_{\theta}(s)$ in \eqref{eq:bellman op} is not smooth in $\theta$, which is troublesome in both algorithm design and analysis. To solve this issue, we introduce the following smoothed robust Bellman operator $\hat{\mathbf T}_{\pi}$ by smoothing the max with a LSE(LogSumExp):
\begin{align}
    \hat{\mathbf T}_{\pi}V(s)=\mE_{A\sim\pi(\cdot|s)}\left[c(s,A)+\gamma(1-R)\sum_{s'\in\mcs} p^{A}_{s,s'} V(s')+\gamma R\cdot \text{LSE}(V)\right],
\end{align}
where $\text{LSE}(V)=\frac{\log \left(\sum_s e^{\varrho V(s)}\right)}{\varrho}$ is the LogSumExp w.r.t. $V$ with a parameter $\varrho >0$. 
Note that when $\varrho \to \infty$, the smoothed robust Bellman operator $\hat{\mathbf T}_{\pi} \to \mathbf T_{\pi}$. 
The LSE operator can also be replaced by some other operator that approximates the max operator and is smooth, e.g., mellow-max \citep{Asadi2016}.
In the following, we first show that the fixed point of 
$\hat{\mathbf T}_{\pi}$ exists for any $\varrho $, and the fixed points converge to the one of $\mathbf T_{\pi}$ for large $\varrho $.

\begin{theorem}\label{thm:lim}
(1). For any $\varrho $, $\hat{\mathbf T}_{\pi}$ has a fixed point.

(2). Let $V_1$ and $V_2$ be the fixed points of $\hat{\mathbf T}_{\pi}$ and $\mathbf T_{\pi}$, respectively. Then 
\begin{align}
    \|V_1-V_2\|_{\infty}&\leq \frac{\gamma R  }{1-\gamma}\frac{\log |\mcs|}{\varrho} \rightarrow 0, \text{ as }\varrho \rightarrow\infty.
\end{align}

\end{theorem}

We then denote by $J(\theta)$ the smoothed MSPRBE with the LSE operator, and the goal is:
\begin{align}\label{eq:sMSPRBE}
    \min_\theta J(\theta)=\min_\theta\left\| \mathbf\Pi \hat{\mathbf T}_{\pi} V_{\theta}-V_{\theta}\right\|^2_{\mu_\pi}.
\end{align}



\subsection{Algorithm Development}
In the following, we develop the robust TDC algorithm to solve the problem in \eqref{eq:sMSPRBE}. We will first derive the gradient of the smoothed MSPRBE, $J(\theta)$, and then design a two time-scale update rule using the weight doubling trick in \citep{sutton2009fast} to solve the double sampling problem.
Define 
$
    \delta_{s,a,s'}(\theta)\triangleq c(s,a)+\gamma (1-R) V_{\theta}(s')+\gamma R \text{LSE}(V_{\theta}) -V_{\theta}(s),
$
where $\text{LSE}(V_{\theta})$ is the LogSumExp function w.r.t. $V_\theta=\theta^\top\phi$. 
Denote by $C\triangleq \mathbb{E}_{\mu_\pi}\left[\phi_S^\top\phi_S\right]$. 
Then, 
$
    \mathbb{E}_{\mu}[\delta_{S,A,S'}(\theta)\phi_S]
    =\Phi^\top D\left(\hat{\mathbf T}_{\pi}V_{\theta}-V_{\theta}\right),
$
where $D=\text{diag}(\mu_\pi(s_1),\mu_\pi(s_2),...,\mu_\pi(s_{|\mcs|}))$ and $\Phi=(\phi_{s_1},\phi_{s_2},...,\phi_{s_{|\mcs|}})^\top \in \mathbb R^{|\mcs|\times N}$.
We know that 
$
    \mathbf{\Pi}^\top D\mathbf{\Pi}=D^\top\Phi(\Phi^\top D\Phi)^{-1}\Phi^\top D
$ from \citep{maei2011gradient}. Hence we have
\begin{align}
    J(\theta)&=\left\| \mathbf\Pi \hat{\mathbf T}_{\pi} V_{\theta}-V_{\theta}\right\|^2_{\mu_\pi}
    =\mathbb{E}_{\mu_\pi}[\delta_{S,A,S'}(\theta)\phi_S]^\top C^{-1}\mathbb{E}_{\mu_\pi}[\delta_{S,A,S'}(\theta)\phi_S].
\end{align}
 Then, its gradient can be written as:
\begin{align*}
    &-\frac{1}{2}\nabla J(\theta)=-\mathbb{E}_{\mu_\pi}[(\nabla \delta_{S,A,S'}(\theta)) \phi_S]^\top C^{-1}\mathbb{E}_{\mu_\pi}[\delta_{S,A,S'}(\theta)\phi_S]\nn\\
    &=\mE_{\mu_\pi}[\delta_{S,A,S'}(\theta)\phi_S]-\gamma \mE_{\mu_\pi}\bigg[\bigg(  (1-R)\phi_{S'}+ R\cdot \nabla \text{LSE}(V_{\theta})\bigg) \phi_S^\top\bigg]\omega(\theta),
\end{align*}
where $\omega(\theta)= C^{-1}\mathbb{E}_{\mu_\pi}[\delta_{S,A,S'}(\theta)\phi_S]$. It can be seen that to obtain an unbiased estimate of $\nabla J(\theta)$, two independent samples are needed as there exists a multiplication of two expectations, which is not applicable when there is only one sample trajectory. 
We then utilize the weight doubling trick in \citep{sutton2009fast}, and design the robust TDC algorithm in \Cref{alg:TDC}. Specifically, we introduce a fast time scale to estimate $\omega(\theta)$, and a slow time scale to estimate $\nabla J(\theta)$. Denote the projection by $\mathbf \Pi_K(x)\triangleq \arg\min_{\|y\|\leq K} \|y-x\|$ for any $x\in \mathbb R^N$. 
Our robust TDC algorithm in \Cref{alg:TDC} can be implemented in an online and incremental fashion.
If the uncertainty set becomes a singleton, i.e., $R=0$, then \Cref{alg:TDC} reduces to the vanilla TDC algorithm.
\begin{algorithm}
\caption{Robust TDC with Linear Function Approximation}
\label{alg:TDC}
\textbf{Input}:   $T$,$\alpha$, $\beta$, $\varrho$, $\phi_i$ for $i=1,...,N$, projection radius  $K$\\
\textbf{Initialization}: $\theta_0$,$w_0$, $s_0$
\begin{algorithmic}[1] 
\STATE {Choose $W\sim \text{Uniform}(0,1,...,T-1)$}
\FOR {$t=0,1,2,...,W-1$}
		\STATE Take action according to $\pi(\cdot|s_t)$ and observe $s_{t+1}$ and $c_{t}$
		\STATE $\phi_t\leftarrow \phi_{s_t}$
		\STATE $\delta_{t}(\theta_{t})\leftarrow c_t+\gamma (1-R)V_{\theta_t}(s_{t+1})+\gamma R \frac{\log (\sum_s e^{\varrho \theta^\top\phi_s})}{\varrho}-V_{\theta_t}(s_t) $
		\STATE $\theta_{t+1} \leftarrow  \mathbf \Pi_K \left(\theta_{t}+\alpha\left(\delta_t(\theta_t)\phi_t-\gamma\bigg( (1-R)\phi_{t+1}+ R \sum\limits_{s\in\mcs} \left(\frac{e^{\varrho V_{\theta}(s)}\phi_s}{\sum_{j\in\mcs}e^{\varrho V_{\theta}(j)}}\right)\bigg)\phi_t^\top\omega_t \right)\right)$
		\STATE $\omega_{t+1} \leftarrow \mathbf{\Pi}_K(\omega_{t}+\beta(\delta_{t}(\theta_{t})-\phi_{t}^\top  \omega_{t})\phi_{t})$
		\ENDFOR
\end{algorithmic}
\textbf{Output}: $\theta_W$
\end{algorithm}

\subsection{Finite-Time Error Bound of Robust TDC}

Unlike the vanilla TDC algorithm,  $J(\theta)$ here is non-convex. Therefore, we are interested in the convergence to stationary points, i.e., the rate of $\|\nabla J(\theta)\|\to 0$. We first make some standard assumptions which are commonly used in RL algorithm analysis, e.g., \citep{wang2020finite,kaledin2020finite,xu2019two,srikant2019,bhandari2018finite}.
\begin{assumption}[Bounded feature]
  $\|\phi_{s}\|_2\leq 1, \forall s\in\mcs$.
\end{assumption}
\begin{assumption}[Bounded cost function]
  $|c(s,a)|\leq c_{\max}, \forall s\in\mcs$ and $a\in\mca$.
\end{assumption}
\begin{assumption}[Problem solvability]
 The matrix $C=\mathbb{E}_{\mu_\pi}[\phi_S\phi_S^\top  ]$ is non-singular with $\lambda>0$ being its smallest eigenvalue.
\end{assumption} 
\begin{assumption}[Geometric uniform ergodicity]\label{ass:1}
 There exist some constants $m>0$ and $\rho \in (0,1)$ such that  for any $t>0$,
$
    \max_{s\in\mathcal{S}} d_{\text{TV}}(\mathbb{P}(s_t |s_0=s),  \mu_\pi) \leq m\rho^t .
$
\end{assumption}

In the following theorem, we characterize the finite-time error bound for the convergence of our robust TDC algorithm. Here we only provide the order of the bounds in terms of $T$. The explicit bounds can be found in \eqref{eq:tdcbound} in \Cref{sec:finitebound}. 
\begin{theorem}\label{thm:TDC}
Consider the following step-sizes: $\beta= \mathcal{O}\left(\frac{1}{T^b}\right)$, and $\alpha= \mathcal{O}\left(\frac{1}{T^a}\right)$, where $\frac{1}{2}< a\leq 1$ and $0<b\leq a$. Then we have that
\begin{align}\label{eq:theorembound}
    \mathbb{E}[\|\nabla  J(\theta_W)\|^2]=\mathcal{O}\left(\frac{1}{T\alpha}+\alpha\log(1/\alpha)+\frac{1}{T\beta} +\beta\log(1/ \beta)\right),
\end{align}
If we further let $a=b=0.5$, then 
$
    \mathbb{E}[\|\nabla  J(\theta_W)\|^2]=\mathcal{O}\left(\frac{\log T}{\sqrt{T}}\right).
$
\end{theorem}


The robust TDC has a matching complexity with the vanilla TDC with non-linear function approximation \citep{wang2021finite}, but provides the additional robustness to model uncertainty. 
It does not need to relax the uncertainty set like in \citep{roy2017reinforcement}, and our convergence results do not need a condition on the discount factor.

\section{Experiments}
\subsection{Robust Q-Learning}\label{sec:6.1}
In this section, we compare our robust Q-learning with the vanilla non-robust Q-learning. We use OpenAI gym framework \citep{brockman2016openai}, and consider two different problems: Frozen lake and Cart-Pole. One more example of the taxi problem is given in the appendix. To demonstrate the robustness, the policy is learned in a perturbed MDP, and is then tested on the true unperturbed MDP. Specifically, during the training, we set a probability $p$ such that after the agent takes an action, with probability $p$, the state transition is uniformly over $\mcs$, and with probability $1-p$ the state transition is according to the true unperturbed transition kernel. 
The behavior policy for all the experiments below is set to be a uniform distribution over the action space given any state, i.e., $\pi_b(a|s)=\frac{1}{|\mca|}$ for any $s\in\mcs$ and $a\in\mca$. 
We then evaluate the performance of the obtained policy in the unperturbed environment.  At each time t, the policy we evaluate is the greedy-policy w.r.t. the current estimate of the Q-function, i.e., $\pi_t(s)=\arg\max_a Q_t(s,a)$. 
A Monte-Carlo method with horizon 100 is used to evaluate the accumulated discounted reward of the learned policy on the unperturbed MDP. We take the average over 30 trajectories. 
More details are provided in the appendix. 

%

\begin{figure}[htb]
\centering 
\subfigure[ p=0.1, R=0.1]{
\label{Fig.fl1}
\includegraphics[width=0.31\linewidth]{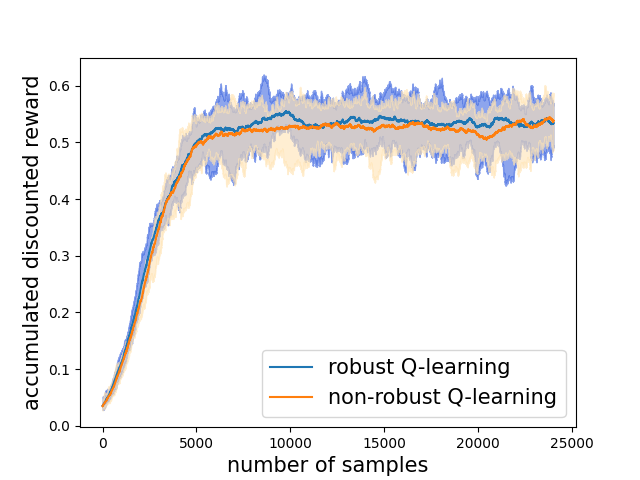}}
\subfigure[ p=0.05, R=0.2]{
\label{Fig.fl2}
\includegraphics[width=0.31\linewidth]{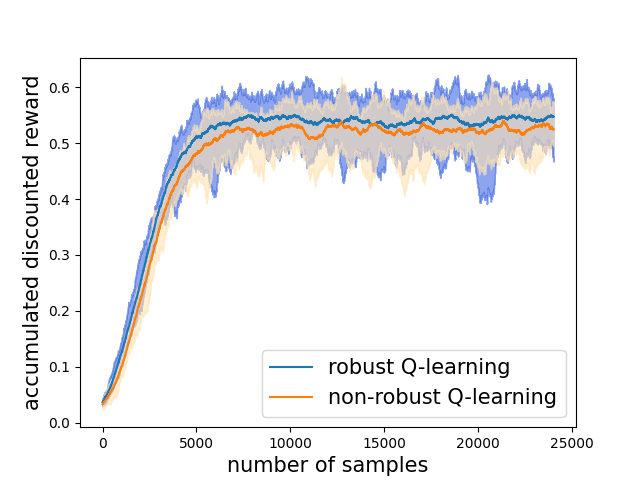}}
\subfigure[ p=0.1, R=0.2]{
\label{Fig.fl3}
\includegraphics[width=0.31\linewidth]{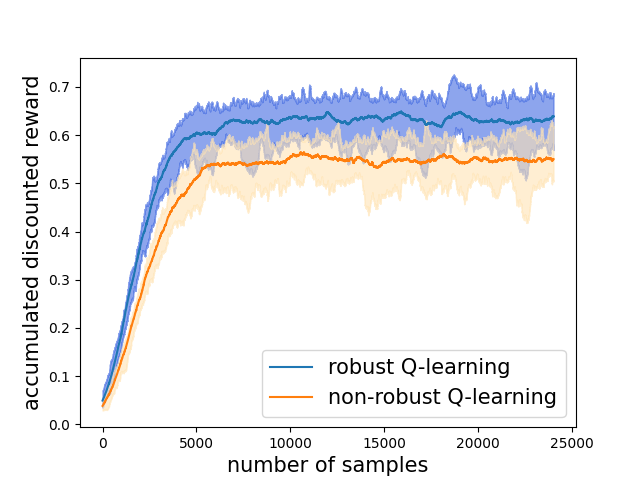}}
\captionsetup{font={normalsize}}
\caption{\textbf{FrozenLake-v0}: robust Q-learning v.s. non-robust Q-learning.}
\label{Fig.44fl}
\end{figure}
\begin{figure}[htb]
\centering 
\subfigure[ p=0.1, R=0.1]{
\label{Fig.cp1}
\includegraphics[width=0.312\linewidth]{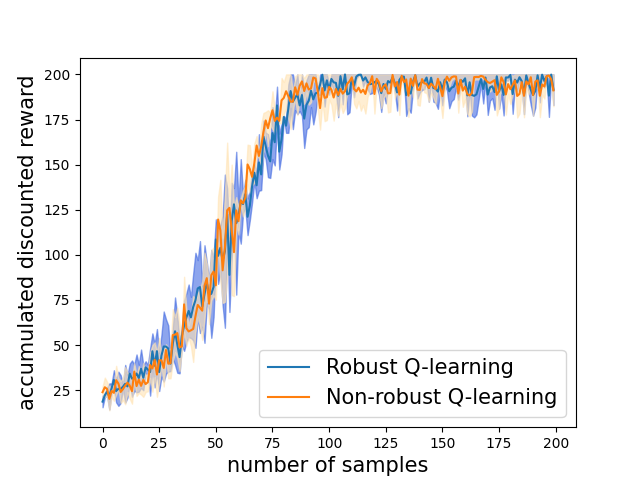}}
\subfigure[ p=0.05, R=0.2]{
\label{Fig.cp2}
\includegraphics[width=0.312\linewidth]{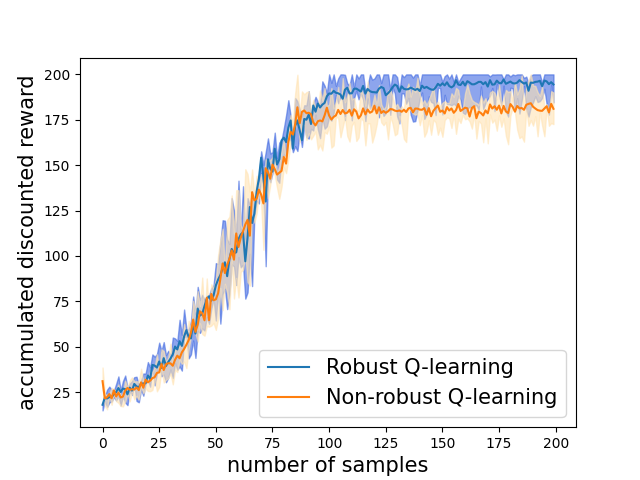}}
\subfigure[ p=0.1, R=0.2]{
\label{Fig.cp3}
\includegraphics[width=0.312\linewidth]{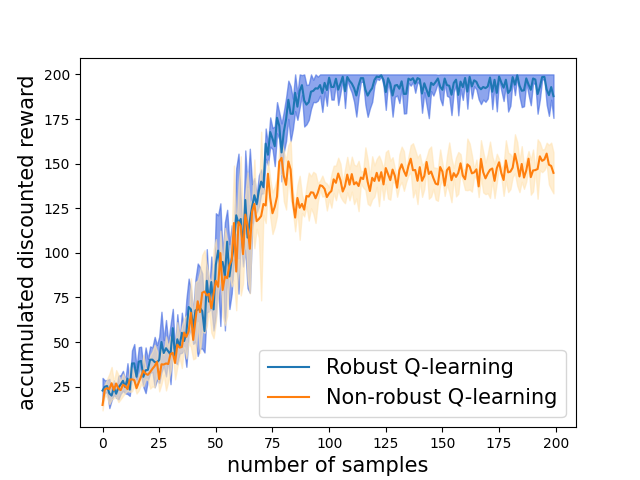}}
\captionsetup{font={normalsize}}
\caption{\textbf{CartPole-v0}: robust Q-learning v.s. non-robust Q-learning.}
\label{Fig.CP}
\end{figure}

In \Cref{Fig.44fl} and \Cref{Fig.CP}, we plot the accumulated discounted reward of both algorithms under different $p$ and $R$ for both problems. The upper and lower envelopes of the curves correspond to the 95 and 5 percentiles of the 30 trajectories, respectively. 
It can be seen that overall our robust Q-learning algorithm achieves a much higher reward than the vanilla Q-learning. This demonstrates the robustness of our robust Q-learning algorithm to model uncertainty. Moreover, as $p$ and $R$ getting larger, i.e., as the MDP that we learn the policy deviates from the MDP we test the policy, the advantage of our robust Q-learning algorithm is getting more significant compared to the vanilla Q-learning algorithm. 

\subsection{Robust TDC with Linear Function Approximation}\label{sec:6.2}
In this section we compare our robust TDC with the vanilla non-robust TDC with linear function approximation on the $4\times 4$ Frozen Lake problem. The problem setting is the same as the one in Section \ref{sec:6.1}. More details about the experiment setup are provided in the appendix. 

We implement the two algorithms using samples from the perturbed MDP both for 30 times, and obtain 30 sequences of $\{\theta^i_t\}_{t=1}^\infty$, $i=1,2,...,30$. We then compute the squared gradient  norm $\| \nabla J(\theta)\|^2$ on the true unperturbed MDP, and see whether $\{\theta^i_t\}_{t=1}^\infty$ converges to some stationary points on the true unperturbed MDP. In Fig. \ref{Fig.TDC}, we plot the average squared gradient norm $\| \nabla J(\theta)\|^2$ for different $p$ and $R$. The upper and lower envelops are   the 95 and 5 percentiles of the 30 curves.
It can be seen that our robust TDC converges much faster than vanilla TDC, and as the model mismatch between the training and test MDPs enlarges, the vanilla TDC may diverge (Fig. \ref{Fig.tdc3}), while our robust TDC still converges to some stationary point. Also, the robust TDC has a much smaller variance, which indicates a much stable behavior  under model uncertainty. 
\begin{figure}[htb]
\centering 
\subfigure[ p=0.1, R=0.1]{
\label{Fig.tdc1}
\vspace{-1cm}\includegraphics[width=0.312\linewidth]{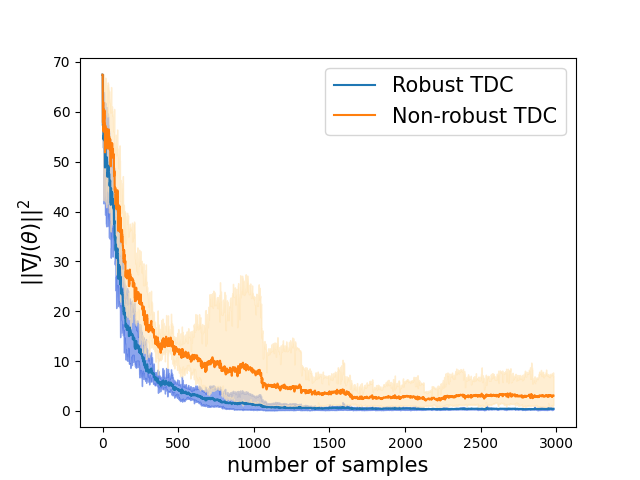}}
\subfigure[ p=0.05, R=0.2]{
\label{Fig.tdc2}
\includegraphics[width=0.312\linewidth]{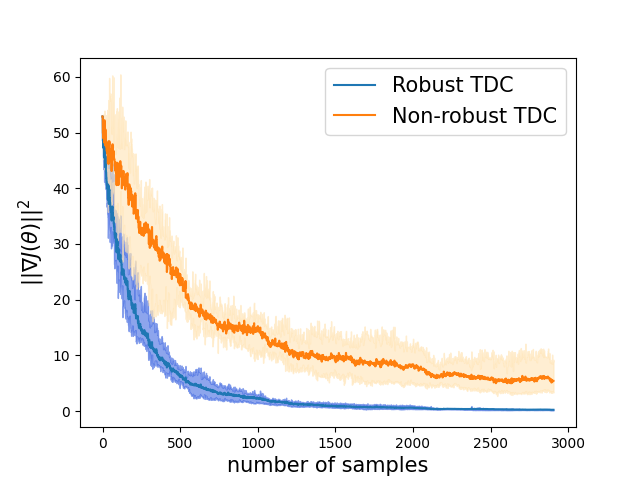}}
\subfigure[ p=0.1, R=0.2]{
\label{Fig.tdc3}
\includegraphics[width=0.312\linewidth]{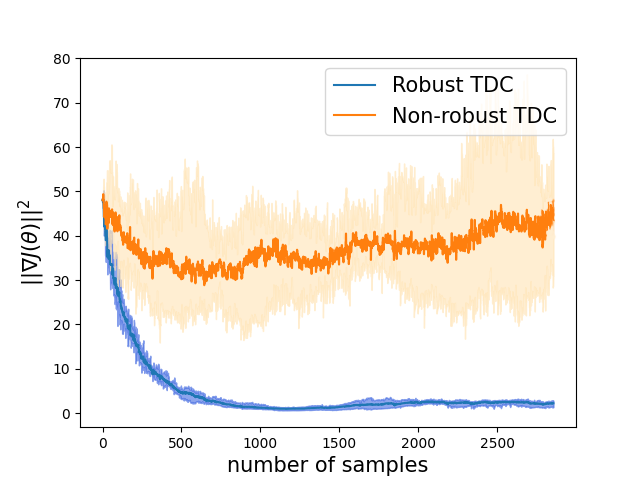}}
\captionsetup{font={normalsize}}
\caption{\textbf{FrozenLake-v0}: $\|\nabla J(\theta)\|^2$ of robust TDC and non-robust TDC.}
\label{Fig.TDC}
\end{figure}

\subsection{Comparison with The Adversarial Training Approach}

We also compare our robust Q-learning with Robust Adversarial Reinforcement Learning (RARL) in \citep{pinto2017robust}. To apply their algorithm to our problem setting, we model the nature as an adversarial player, and its goal is to minimize the reward that the agent receives. The action space $\mca_{ad}$ of the nature is set to be the state space $\mca_{ad}\triangleq \mcs$. Then the perturbed training environment can be viewed as an adversarial model: both the agent and the adversary take actions $a_{a}, a_{ad}$, then the environment will transit to state $a_{ad}$ with probability $R$ and transit following the unperturbed MDP $p_s^{a_a}$ with probability $1-R$. The goal of the maximize its accumulated reward, while the goal of the natural is to minimize it.

Following the RARL algorithm \citep{pinto2017robust}, in each iteration of the training, we first fix the adversarial policy and use Q-learning to optimize the agent's policy and obtain the Q-table $Q_t$. Then we fix the agent's policy and optimize the adversarial policy. 

After each training iteration, we test the performance of the greedy policies w.r.t. Q-tables obtained from robust Q-learning and RARL. The testing environment is set to be the worst-case, i.e., after the agent takes an action, the environment transits to the state which has the minimal value function ($\arg\min_{s\in\mcs} V_t(s)$) with probability $p$. We plot the accumulated discounted rewards of both algorithms against number of training iterations under different parameters. We set $\alpha=0.2$ and $\gamma=0.9$. It can be seen from Fig.~\ref{Fig.AD} that our robust Q-learning achieves a higher accumulative reward, and thus is more robust that the RARL algorithm in \citep{pinto2017robust}. Also our robust Q-learning is more stable during training, i.e.,  the variance is smaller.
\begin{figure}[htb]
\centering 
\subfigure[ p=0.1, R=0.1]{
\label{Fig.ad1}
\includegraphics[width=0.3\linewidth]{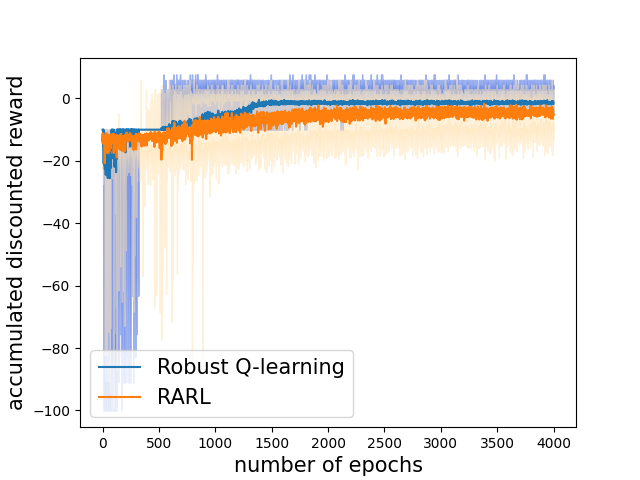}}
\subfigure[ p=0.05, R=0.2]{
\label{Fig.ad2}
\includegraphics[width=0.3\linewidth]{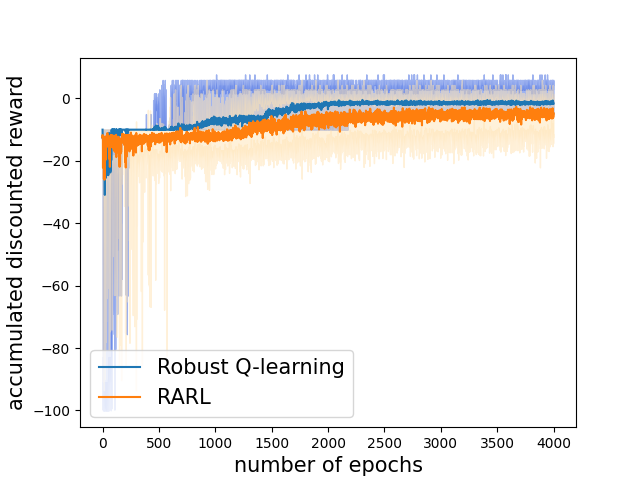}}
\subfigure[ p=0.1, R=0.2]{
\label{Fig.ad3}
\includegraphics[width=0.3\linewidth]{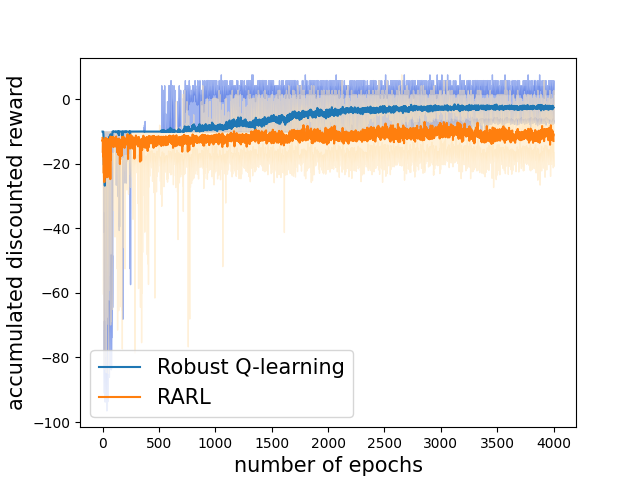}}
\captionsetup{font={normalsize}}
\caption{\textbf{Taxi-v3}: robust Q-learning v.s. RARL.}
\label{Fig.AD}
\end{figure}

 \section{Conclusion}
 \vspace{-0.3cm}
 In this paper,  we develop a novel approach for solving model-free robust RL problems with model uncertainty. Our algorithms can be implemented in an online and incremental fashion, do not require additional memory than their non-robust counterparts. We theoretically proved the convergence of our algorithms under no additional assumption on the discount factor, and further characterized their finite-time error bounds, which match with their non-robust counterparts (within a constant factor). Our approach can be readily extended to robustify TD, SARSA and other GTD algorithms. \textbf{Limitations: }It is also of future interest to investigate robustness to reward uncertainty, and other types of uncertainty sets, e.g., ones defined by KL divergence, Wasserstein distance and total variation. 
 \textbf{Negative societal impact:}
To the best of the authors' knowledge, this study does not have any potential negative impact on the society.

\section{Acknowledgment}
The work of Y. Wang and S. Zou was supported by the National Science Foundation under Grants CCF-2106560 and CCF- 2007783.
\newpage

\bibliography{robust}

\newpage
\appendix
{\Large\textbf{Supplementary Materials}}

\section{Proof of Theorem \ref{thm:Qasyconv}: Asymptotic Convergence of Robust Q-Learning}
In this section we show that the robust Q-learning converges exactly to the optimal robust Q function $Q^*$.
Recall that the optimal robust Q function $Q^*$ is the solution to the robust Bellman operator $\mathbf  T$:
\begin{align}
    Q^*(s,a)=c(s,a)+\gamma\sigma_{\mathcal{P}^a_s}((\min_{a\in\mca} Q^*(s_1,a),\min_{a\in\mca} Q^*(s_2,a),...,\min_{a\in\mca} Q^*(s_{|\mcs|},a))^\top).
\end{align}
It can be shown that the estimated update is an unbiased estimation of $\mathbf T$. More specifically,
\begin{align}\label{eq:exp=0}
    \mathbf TQ(s,a)&=c(s,a)+\gamma \sigma_{\mathcal{P}^a_s}(V)\nn\\
    &=c(s,a)+\gamma (1-R) (p^a_s)^\top V +R \max_{s'} V(s')\nn\\
    &=c(s,a)+\gamma (1-R) \sum_{s'}(p^a_{s,s'})  V(s') +R \max_{s'} V(s')\nn\\
    &=c(s,a)+\gamma   \sum_{s'}p^a_{s,s'} \left((1-R)(\mathbbm{1}_{s'})^\top V + R \max_{q} q^\top V \right),
\end{align}
which is the expectation of the estimated update in line 5 of Algorithm \ref{alg:1}. 
\subsection{Robust Bellman operator is a contraction}
It was shown in \citep{iyengar2005robust, roy2017reinforcement} that the robust Bellman operator is a contraction. Here, for completeness, we include the proof for our R-contamination uncertainty set. More specifically,
\begin{align}
    &|\mathbf TQ(s,a)-\mathbf TQ'(s,a)|\nn\\
    &=|c(s,a)+\gamma \sigma_{\mathcal{P}^a_s}(V) -c(s,a)-\gamma \sigma_{\mathcal{P}^a_s}(V')|\nn\\
    &=\gamma|\sigma_{\mathcal{P}^a_s}(V) - \sigma_{\mathcal{P}^a_s}(V')|\nn\\
    &=\gamma |\max_{q} \left\{(1-R)(p^a_s)^\top V+ R q^\top V \right\}-\max_{q'} \left\{(1-R)(p^a_s)^\top V'+ R q'^\top V' \right\} |\nn\\
    &=\gamma \left|\sum_{s'\in\mcs} p^a_{s,s'} \left( (1-R)V(s')\right)+ R \max_{s'} V(s')-\sum_{s'\in\mcs} p^a_{s,s'} \left((1-R)V'(s')\right)- R \max_{s'} V'(s')\right|\nn\\
    &=\gamma\left| \sum_{s'\in\mcs} p^a_{s,s'}(1-R) \left(V(s')-V'(s')\right)+R(\max_{s'} V(s')-\max_{s'} V'(s'))\right|\nn\\    
    &\leq \gamma\left| \sum_{s'\in\mcs} p^a_{s,s'}(1-R) \left(\min_a Q(s',a)-\min_b Q'(s',b)\right)\right|+\gamma R(|\max_{s'} V(s')-\max_{s'} V'(s')|) \nn\\ 
    &\leq \gamma\sum_{s'\in\mcs} p^a_{s,s'}(1-R)\left|  \left(\min_a Q(s',a)-\min_b Q'(s',b)\right)\right|+\gamma R \max_s |(V(s)-V'(s))|\nn\\
    &\overset{(a)}{\leq} \gamma\sum_{s'\in\mcs} p^a_{s,s'}(1-R)\|Q-Q'\|_{\infty}+\gamma R\|Q-Q'\|_{\infty}\nn\\
    &\leq \gamma \|Q-Q'\|_{\infty},
\end{align}
where $(a)$ can be shown as below. Assume that $a_1=\arg\min_a Q(s',a)$ and $b_1=\arg\min_a Q'(s',a)$. Then if $Q(s',a_1)>Q'(s',b_1)$, then
\begin{align}
    |Q(s',a_1)-Q'(s',b_1)|=Q(s',a_1)-Q'(s',b_1)\leq Q(s',b_1)-Q'(s',b_1)\leq \|Q-Q'\|_{\infty}.
\end{align}
Similarly, it can also be shown when $Q(s',a_1)\leq Q'(s',b_1)$, and hence the inequality $(a)$ holds. 
\subsection{Asymptotic Convergence of Robust Q-Leaning}
With the definition of $\mathbf T$, the update \eqref{eq:Qupdate} of robust Q-learning can be re-written as a stochastic approximation:
\begin{align}\label{eq:QSA}
    Q_{t+1}(s_t,a_t)=(1-\alpha_t)Q_t(s_t,a_t)+\alpha_t(\mathbf TQ_t(s_t,a_t)+\eta_t(s_t,a_t,s_{t+1})),
\end{align}
where the noise term is 
\begin{align}
    \eta_t(s_t,a_t,s_{t+1})=c(s_t,a_t)+\gamma R \max_s V_t(s) +\gamma (1-R) V_{t}(s_{t+1})-\mathbf TQ_t(s_t,a_t).
\end{align}
From \eqref{eq:exp=0}, we have that 
\begin{align}
   \mE[ \eta_t(S_t,A_t,S_{t+1})|S_t=s_t,A_t=a_t]=0.
\end{align}
The variance can be bounded by
\begin{align}
    \mE[ (\eta_t(S_t,A_t,S_{t+1}))^2]&\leq \gamma^2(1-R)^2(\max_{s,a} Q_t^2(s,a)),
\end{align}
where the last inequality is from $V_t(s_{t+1})\leq \max_s V_t(s) \leq \max_{s,a} Q_t(s,a)$. 
Thus the noise term $\eta_t$ has zero mean and bounded variance. From \citep{borkar2000ode}, we know that the stochastic approximation \eqref{eq:QSA} converges to the fixed point of $\mathbf T$, i.e., $Q^*$. Hence we showed that robust Q-learning converges to optimal optimal robust Q function $Q^*$ with probability 1.
\section{Finite-Time Analysis of Robust Q-Learning}
In this section, we develop the finite-time analysis of the Algorithm \ref{alg:1}.
\subsection{Notations}
We first introduce some notations.  For a vector $v=(v_1,v_2,...,v_n)$, we denote the entry wise absolute value $(|v_1|,...,|v_n|)$ by $|v|$.
For a sample $O_t=(s_t,a_t,s_{t+1})$, define $\Lambda_{t+1} \in \mathbb{R}^{|\mcs||\mca|\times|\mcs||\mca|}$ as 
\begin{equation}
  \Lambda_{t+1}((s,a),(s',a'))=\left\{
           \begin{array}{lr}
            \alpha, & \text{ if }(s,a)=(s',a')=(s_t,a_t),\\
             0, &  \text{ otherwise.}
             \end{array}
\right.
\end{equation}
Also we define the sample transition matrix $ P_{t+1}\in\mathbb{R}^{|\mcs||\mca|\times|\mcs|}$ as
\begin{equation}
  P_{t+1}((s,a),s')=\left\{
            \begin{array}{lr}
            1, & \text{ if }(s,a,s')=O_t,\\
             0, &  \text{ otherwise.}
             \end{array}
\right.
\end{equation}
We also define the transition kernel matrix $ P\in\mathbb{R}^{|\mcs||\mca|\times|\mcs|}$ as
\begin{equation}
  P((s,a),s')=p^a_{s,s'}.
\end{equation}
We use $Q_t \in\mathbb{R}^{|\mcs||\mca|}$ and $V_t \in\mathbb{R}^{|\mcs|}$ to denote the vectors of value functions. Denote the cost function $c\in\mathbb{R}^{|\mcs||\mca|}$ with entry $c(s,a)$ being the cost received at $(s,a)$. Then the update of robust Q-learning \eqref{eq:Qupdate} can be written in matrix form as 
\begin{align}\label{eq:matrixQupdate}
    Q_t=(I-\Lambda_t)Q_{t-1}+\Lambda_t\Big(c+\gamma (1-R)P_tV_{t-1}+\gamma R\max_{s\in\mcs}V_{t-1}(s) P_t \textbf{1}\Big),
\end{align}
where $\textbf{1}$ denotes the vector $(1,1,1,...,1)^\top \in\mathbb{R}^{|\mcs|}$. The robust Bellman equation can be written as
\begin{align}\label{eq:Bellman}
    Q^*=c+\gamma(1-R)PV^*+\gamma R\max_{s\in\mcs}V^*(s)P\mathbf{1}. 
\end{align}

\subsection{Analysis}
Define $\psi_t=Q_t-Q^*$, then by \eqref{eq:matrixQupdate} and \eqref{eq:Bellman}, we have that 
\begin{align}\label{eq:recursion1}
    \psi_t&=Q_t-Q^*\nn\\
    &=(I-\Lambda_t)Q_{t-1}+\Lambda_t(c+\gamma (1-R)P_tV_{t-1}+\gamma R\max_{s\in\mcs}V_{t-1}(s) P_t \textbf{1})-Q^*\nn\\
    &=(I-\Lambda_t)(Q_{t-1}-Q^*)+\Lambda_t(c+\gamma (1-R)P_tV_{t-1}+\gamma R\max_{s\in\mcs}V_{t-1}(s) P_t \textbf{1}-Q^*)\nn\\
    &=(I-\Lambda_t)\psi_{t-1}+\Lambda_t(\gamma (1-R)P_tV_{t-1}+\gamma R\max_{s\in\mcs}V_{t-1}(s) P_t \textbf{1}-\gamma(1-R)PV^*\nn\\
    &\quad-\gamma R\max_{s\in\mcs}V^*(s)P\mathbf{1})\nn\\
    &=(I-\Lambda_t)\psi_{t-1}+\gamma(1-R)\Lambda_t\underbrace{(P_tV_{t-1}-PV^*)}_{k_1}\nn\\
    &\quad+\gamma R \Lambda_t\underbrace{(\max_{s\in\mcs}V_{t-1}(s)P_t\mathbf{1}-\max_{s\in\mcs}V^*(s)P\mathbf{1}))}_{k_2}.
\end{align}
The term $k_1$ can be written as
\begin{align}
    P_tV_{t-1}-PV^*= P_tV_{t-1}-P_tV^*+P_tV^*-PV^*=P_t(V_{t-1}-V^*)+(P_t-P)V^*.
\end{align}
Similarly, we have that 
\begin{align}
    k_2=\Big(\max_{s\in\mcs}V_{t-1}(s)-\max_{s\in\mcs}V^*(s)\Big)P_t\mathbf{1}+\max_{s\in\mcs}V^*(s)(P_t-P)\mathbf{1}. 
\end{align}
Hence \eqref{eq:recursion1} can be  written as 
\begin{align}\label{eq:recursion2}
    \psi_t&=Q_t-Q^*\nn\\
    &=(I-\Lambda_t)\psi_{t-1}+\gamma(1-R)\Lambda_t(P_t(V_{t-1}-V^*)+(P_t-P)V^*)\nn\\
    &\quad+\gamma R \Lambda_t\left(\left(\max_{s\in\mcs}V_{t-1}(s)-\max_{s\in\mcs}V^*(s)\right)P_t\mathbf{1}+\max_{s\in\mcs}V^*(s)(P_t-P)\mathbf{1}\right)\nn\\
    &=(I-\Lambda_t)\psi_{t-1}+\left(\gamma(1-R)\Lambda_t(P_t-P)V^*) +\gamma R \Lambda_t(\max_{s\in\mcs}V^*(s)(P_t-P)\mathbf{1})\right)\nn\\
    &\quad+\left(\gamma(1-R)\Lambda_t(P_t(V_{t-1}-V^*))+\gamma R \Lambda_t\left(\left(\max_{s\in\mcs}V_{t-1}(s)-\max_{s\in\mcs}V^*(s)\right)P_t\mathbf{1}\right)\right).
\end{align}
By  applying \eqref{eq:recursion2} recursively, we have that 
\begin{small}
\begin{align}
    \psi_t&=\underbrace{\prod^t_{j=1}(I-\Lambda_j)\psi_0}_{k_{1,t}}\nn\\
    &\quad+\underbrace{\gamma(1-R)\sum^t_{i=1}\prod^t_{j=i+1} (I-\Lambda_j)\Lambda_i(P_i-P)V^*+\gamma R \sum^t_{i=1}\prod^t_{j=i+1} (I-\Lambda_j)\Lambda_i\max_{s\in\mcs}V^*(s)(P_i-P)\mathbf{1} }_{k_{2,t}}\nn\\
    &\quad+\underbrace{\gamma(1-R)\sum^t_{i=1}\prod^t_{j=i+1} (I-\Lambda_j)\Lambda_i P_i(V_{i-1}-V^*)+\gamma R \sum^t_{i=1}\prod^t_{j=i+1} (I-\Lambda_j)\Lambda_i (\max_{s\in\mcs}V_{i-1}(s)-\max_{s\in\mcs}V^*(s)) P_i\mathbf{1}}_{k_{3,t}}.
\end{align}
\end{small}
We then bound terms $k_{i,t}$ separately.

\begin{Lemma}
Define $t_{\text{frame}}=\frac{443 t_{\text{mix}}}{\mu_{\min}}\log\frac{4|\mcs||\mca|T}{\delta}$. Then with probability at least $1-\delta$, for any $(s,a)\in\mcs\times\mca$ and any $t\geq t_{\text{frame}}$, $k_{1,t}$ can be bounded as 
\begin{align}
    |k_{1,t}|\leq (1-\alpha)^{\frac{t\mu_{\min}}{2}}\|\psi_0\|_{\infty}\mathbf{1};
\end{align}
and for $t<t_{\text{frame}}$, 
\begin{align}
|k_{1,t}|\leq \|\psi_0\|_{\infty}\mathbf{1}. 
\end{align}
\end{Lemma}
\begin{proof}
First note that the $(s,a)$-entry of $k_{1,t}$ can be written as
\begin{align}\label{eq:k1t}
    k_{1,t}(s,a)=(1-\alpha)^{K_t(s,a)} \psi_0(s,a),
\end{align}
where $K_t(s,a)$ denotes the times that the sample trajectory visits $(s,a)$ before the time step $t$. We introduce a lemma from \citep{li2020sample} first:
\begin{Lemma}(Lemma 5 \citep{li2020sample})
For a time-homogeneous and uniformly ergodic Markov chain with state space $\mathcal{X}$ and any $0<\delta<1$, if $t\geq \frac{443 \tmix}{\mu_{\min}}\log \frac{|\mathcal{X}|}{\delta}$, then for any $y\in\mathcal{X}$, 
\begin{align}
    \mathbb{P}_{X_1=y} \left \{ \exists x\in\mathcal{X}: \sum^t_{j=1} \mathbbm{1}{X_j=x}\leq \frac{t\mu(x)}{2} \right\}\leq \delta,
\end{align}
where $\tmix=\min \left\{ t: \max_{x\in\mathcal{X}} d_{\text{TV}}(\mu,P^t(\cdot|x) ) \leq \frac{1}{4}\right\}$; $\mu$ is the stationary distribution of the Markov chain, and $\mu_{\min}\triangleq\min_{x\in\mathcal{X}} \mu(x)$.
\end{Lemma}
From this lemma, we know that  for any $(s,a)\in\mcs\times\mca$ and any $t\geq \frac{443 t_{\text{mix}}}{\mu_{\min}}\log\frac{4|\mcs||\mca|T}{\delta}$, we have that 
\begin{align}
    K_t(s,a)\geq \frac{t\mu_{\min}}{2},
\end{align}
with probability at least $1-\delta$. 

Thus \eqref{eq:k1t} can be bounded as
\begin{align}
    |k_{1,t}(s,a)|\leq (1-\alpha)^{\frac{t\mu_{\min}}{2}} |\psi_0(s,a)|
\end{align}
with probability at least $1-\delta$ for any $(s,a)\in\mcs\times\mca$ and any $t\geq \frac{443 t_{\text{mix}}}{\mu_{\min}}\log\frac{4|\mcs||\mca|T}{\delta}$, which shows the claim.

For $t<t_{\text{frame}}$, the bound is obvious by noting that $\|I-\Lambda_j\|\leq 1$.  
\end{proof}

\begin{Lemma}
There exists some constant $\hat{c}$, such that for any $\delta<1$ and any $t\leq T$ that satisfies $0<\alpha\log\frac{|\mcs||\mca|T}{\delta}<1$, with probability at least $1-\frac{\delta}{|\mcs||\mca|T}$, 
\begin{align}
    \left|k_{2,t}\right|
    \leq 5\gamma\hat{c}\sqrt{\alpha\log\frac{T|\mcs||\mca|}{\delta} } \|V^*(s)\|_{\infty}\mathbf{1},
\end{align}
\end{Lemma}
\begin{proof}
Recall that 
\begin{align}
    k_{2,t}=\gamma(1-R)\sum^t_{i=1}\prod^t_{j=i+1} (I-\Lambda_j)\Lambda_i(P_i-P)V^*+\gamma R \sum^t_{i=1}\prod^t_{j=i+1} (I-\Lambda_j)\Lambda_i(P_i-P)w^*,
\end{align}
where $w^*\triangleq\max_{s\in\mcs}V^*(s)\mathbf{1}$. Then the $(s,a)$-th entry of $k_{2,t}$ can be written as
\begin{align}
    k_{2,t}(s,a)&=\gamma(1-R)\sum^{K_t(s,a)}_{i=1} \alpha(1-\alpha)^{K_t(s,a)-i}(P_{t_i+1}(s,a)-P(s,a))V^*\nn\\
    &\quad+\gamma R\sum^{K_t(s,a)}_{i=1} \alpha(1-\alpha)^{K_t(s,a)-i}(P_{t_i+1}(s,a)-P(s,a))w^*,
\end{align}
where $t_i(s,a)$ is the time step when the trajectory visits $(s,a)$ for the $i$-th time.
We define $\text{Var}_P(V)\in\mathbb{R}^{|\mcs||\mca|}$ being a vector, where $\text{Var}_P(V) (s,a)=\sum_{s'\in \mcs} p^a_{s,s'}(V(s')^2)-(\sum_{s'\in \mcs} p^a_{s,s'}V(s'))^2\triangleq \text{Var}_{P^a_s}[V]$ for any $V\in\mathbb{R}^{|\mcs|}$.

From Section E.1 in \citep{li2020sample}, we know that 
\begin{align}
    \text{Var}\left[ \sum^{K}_{i=1} \alpha(1-\alpha)^{K-i}(P_{t_i+1}(s,a)-P(s,a))V^*\right]=\alpha \text{Var}_{P^a_s}[V^*]\triangleq\sigma_K^2
\end{align}
for some constant $\sigma_K^2$ and any $K\leq T$. Moreover, note that 
\begin{align}
    &\text{Var}\left[\sum^{K}_{i=1} \alpha(1-\alpha)^{K-i}(P_{t_i+1}(s,a)-P(s,a))w^*\right]\nn\\
    &\overset{(a)}{=}\sum^{K}_{i=1} \alpha^2(1-\alpha)^{2K-2i}\text{Var}[(P_{t_i+1}(s,a)-P(s,a))w^*]\nn\\
    &\overset{(b)}{=}\sum^{K}_{i=1} \alpha^2(1-\alpha)^{2K-2i}\text{Var}[\max_{s}V^*(s) ((P_{t_i+1}(s,a)-P(s,a))\mathbf{1})]\nn\\
    &=0,
\end{align}
where equation $(a)$ is due to the fact that $\left\{ P_{t_{1}+1}(s,a), P_{t_{2}+1}(s,a),...,P_{t_{i}+1}(s,a) \right\}_{i\in \mathbb N}$ are independent (Equation $(101)$ in \citep{li2020sample}), $(b)$ is from the definition of $\omega^*$, and the last equation is because the sum of each entries of $P_{t_i+1}(s,a)-P(s,a)$ is $0$.

the last equality is due to the fact that every entries of $w^*$ are the same and hence $\text{Var}_{P^a_s}[w^*]=0$. 

Additionally, we have that 
\begin{align}
    \left\|\alpha(1-\alpha)^{K-i}(P_{t_i+1}(s,a)-P(s,a))V^*\right\|_{\infty} \leq 2\alpha \|V^*(s)\|_{\infty}\triangleq D,
\end{align}
where we denote the bound by $D$. Also,  
\begin{align}
    \left\| \alpha(1-\alpha)^{K-i}(P_{t_i+1}(s,a)-P(s,a))w^*\right\|_{\infty}\leq D.
\end{align}
Hence from the Bernstein inequality (\citep{li2020sample}), we have that 
\begin{align}
    &\left|k_{2,t}(s,a)\right|\nn\\
    &\leq \gamma (1-R) \hat{c}\left(\sqrt{\sigma_K^2 \log\left( \frac{T|\mcs||\mca|}{\delta}\right)}+ D\log\frac{T|\mcs||\mca|}{\delta} \right)+\gamma R \hat{c}\left(D\log\frac{T|\mcs||\mca|}{\delta} \right)\nn\\
    &\leq 5\gamma\hat{c}\sqrt{\alpha\log\frac{T|\mcs||\mca|}{\delta} } \|V^*(s)\|_{\infty},
\end{align}
for some constant $\hat{c}$  with probability at least $1-\frac{\delta}{|\mcs||\mca|T}$, and the last step is due to the fact that $\text{Var}_{P^a_s}[V^*]\leq  \|V^*\|_{\infty}^2$ and $\alpha\log\frac{|\mcs||\mca|T}{\delta}<1$. This hence completes the proof. 
\end{proof}

\begin{Lemma}
For any $t\geq T$, \begin{align}
    |k_{3,t}|\leq \gamma \sum^t_{i=1}\|\psi_{i-1}\|_{\infty}\prod^t_{j=i+1}(I-\Lambda_j)(\Lambda_i)\mathbf{1}. 
\end{align}
\end{Lemma}
\begin{proof}
First note that for any $i$, 
\begin{align}
    \|P_i(V_{i-1}-V^*)\|_{\infty}\leq \|P_i\|_1 \|V_{i-1}-V^*\|_{\infty}=\|V_{i-1}-V^*\|_{\infty}\leq \|\psi_{i-1}\|_{\infty},
\end{align}
where the last inequality is from
\begin{align}
    &\|V_{i-1}-V^*\|_{\infty}=\max_{s} |V_{i-1}(s)-V^*(s)|=|V_{i-1}(s^*)-V^*(s^*)|\nn\\
    &=|\min_a Q_{i-1}(s^*,a)-\min_b Q^*(s^*,b)|\leq \|Q_{i-1}-Q^*\|_{\infty},
\end{align}
where $s^*=\arg\max |V_{i-1}(s)-V^*(s)|$.
Similarly,
\begin{align}
    \left\|(\max_{s\in\mcs}V_{i-1}(s)-\max_{s\in\mcs}V^*(s))P_i\mathbf{1}\right\|_{\infty}\leq |\max_{s\in\mcs}V_{i-1}(s)-\max_{s\in\mcs}V^*(s)|\leq \|\psi_{i-1}\|_{\infty},
\end{align}
where the last inequality is from $|\max_{s\in\mcs}V_{i-1}(s)-\max_{s\in\mcs}V^*(s)|\leq \| V_{i-1}-V^*\|_{\infty}\leq \|Q_{i-1}-Q^*\|_{\infty}$. 
Hence $K_{3,t}$ can be bounded as
\begin{align}
    |k_{3,t}|\leq \gamma \sum^t_{i=1}\|\psi_{i-1}\|_{\infty}\prod^t_{j=i+1}(I-\Lambda_j)(\Lambda_i)\mathbf{1}. 
\end{align}
\end{proof}

Now combine the bounds for terms $k_{1,t}, k_{2,t}$ and $k_{3,t}$, we have the bound on $\psi_t$ as follows. 

For $t<t_{\text{frame}}$, we have that
\begin{align}\label{eq:roughbound1}
    \|\psi_t\|_{\infty}&\leq \|\psi_0\|_{\infty}\mathbf{1}+ 5\gamma\hat{c}\sqrt{\alpha\log\frac{T|\mcs||\mca|}{\delta} } \|V^*(s)\|_{\infty}\mathbf{1}\nn\\
    &\quad+\gamma \sum^t_{i=1}\|\psi_{i-1}\|_{\infty}\prod^t_{j=i+1}(I-\Lambda_j)(\Lambda_i)\mathbf{1};
\end{align}
and for $t\geq t_{\text{frame}}$, we have that 
\begin{align}\label{eq:roughbound2}
    \|\psi_t\|_{\infty}&\leq  (1-\alpha)^{\frac{t\mu_{\min}}{2}}\|\psi_0\|_{\infty}\mathbf{1}+ 5\gamma\hat{c}\sqrt{\alpha\log\frac{T|\mcs||\mca|}{\delta} } \|V^*(s)\|_{\infty}\mathbf{1}\nn\\
    &\quad+\gamma \sum^t_{i=1}\|\psi_{i-1}\|_{\infty}\prod^t_{j=i+1}(I-\Lambda_j)(\Lambda_i)\mathbf{1}.
\end{align}
This bound exactly matches the bound in Equation (42) in \citep{li2020sample} and hence the remaining proof for Theorem \ref{thm:Qbound} can be obtained by following the proof in \citep{li2020sample}. We omit the remaining proof and only state the result.
\begin{theorem}
Define 
\begin{align}
    t_{\text{th}}&=\max\left\{ \frac{2\log\frac{1}{(1-\gamma)^2\epsilon}}{\alpha\mu_{\min}}, t_{\text{frame}}\right\};\\
    \mu_{\text{frame}}&=\frac{1}{2}\mu_{\min}t_{\text{frame}};\\
    \rho&=(1-\gamma)(1-(1-\alpha)^{\mu_{\text{frame}}}),
\end{align}
then for any $\delta<1$ and any $\epsilon<\frac{1}{1-\gamma}$, there exists a universal constant $\hat{c}$ and $c_0$ (determined by $\hat{c}$), such that with probability at least $1-6\delta$, the following bound holds for any $t<T$:
\begin{align}\label{eq:thm5}
    \|Q_t-Q^*\|_{\infty}\leq \frac{(1-\rho)^k\|Q_0-Q^*\|_{\infty}}{1-\gamma}+\frac{5\hat{c}\gamma}{1-\gamma}\sqrt{\alpha \log \frac{|\mcs||\mca|T}{\delta}}+\epsilon,
\end{align}
where $k=\max\left\{ 0, \lfloor \frac{t-t_{\text{th}}}{t_{\text{frame}}} \rfloor\right\}$, as long as 
\begin{align*}
    T\geq c_0\left(\frac{1}{\mu_{\text{min}}(1-\gamma)^5\epsilon^2}+\frac{\tmix}{\mu_{\text{min}}(1-\gamma)} \right)\log \left(\frac{T|\mcs||\mca|}{\delta}\right)\log\left( \frac{1}{\epsilon(1-\gamma)^2}\right),
\end{align*}
and step size $0<\alpha\log\left(\frac{|\mcs||\mca|T}{\delta} \right)<1$.
\end{theorem}
This theorem implies that the convergence rate of our robust Q-learning is as fast as the one of the vanilla Q-learning algorithm in \citep{li2020sample}(except the constant $\hat{c}$).

Finally, to show Theorem \ref{thm:Qbound}, we only need to show each term in \eqref{eq:thm5} is smaller than $\epsilon$. It can be verified that there exists constants $c_1$, such that if we choose the step size 
$\alpha=\frac{c_1}{\log \left(\frac{T|\mcs||\mca|}{\delta}\right)}\min \left(\frac{1}{\tmix},\frac{\epsilon^2(1-\gamma)^4}{\gamma^2} \right)$, then   $\frac{(1-\rho)^k\|Q_0-Q^*\|_{\infty}}{1-\gamma}\leq \epsilon$ (inequality (51) in \citep{li2020sample}) and $\frac{5\hat{c}\gamma}{1-\gamma}\sqrt{\alpha \log \frac{|\mcs||\mca|T}{\delta}}\leq \epsilon$ (by choosing suitable constant $c_1$). Then we have that $\|Q_t-Q^*\|_{\infty}\leq 3\epsilon$. This completes the proof.

\section{Proof of Theorem \ref{thm:lim}: Approximation of Smoothing Robust Bellman Operator}
In this section we prove Theorem \ref{thm:lim}.
First note that for any $x,y \in \mathbb{R}^{|\mcs|}$, 
\begin{align}
    |\text{LSE}(x)-\text{LSE}(y)|\leq \sup_{t\in [0,1]} \|\nabla \text{LSE}(tx+(1-t)y) \|_{1}\|x-y\|_{\infty}.
\end{align}
It can be shown that the gradient of LSE is softmax, i.e., 
\begin{align}
    \frac{\partial \text{LSE}(x)}{\partial x_i}=\frac{e^{\varrho x_i}}{\sum_{j}e^{\varrho x_j}}.
\end{align}
Hence 
\begin{align}
    \|\nabla \text{LSE}(z) \|_{1}=1, \forall z\in \mathbb{R}^{|\mcs|},
\end{align}
which implies that $ |\text{LSE}(x)-\text{LSE}(y)|\leq \|x-y\|_{\infty}.$ Hence for any $x,y\in\mathbb{R}^{|\mcs|}$, we have that
\begin{align}
|\hat{\mathbf T}_{\pi}x (s)-\hat{\mathbf T}_{\pi}y(s)|&=\left|\mathbb{E}_{A}\left[\gamma (1-R)\sum_{s'\in\mcs}p^A_{s,s'}(x(s')-y(s'))+\gamma R (\text{LSE}(x)-\text{LSE}(y)) \right]\right|\nn\\
&\leq \gamma (1-R) \|x-y\|_{\infty}+\gamma R \|x-y\|_{\infty}\nn\\
&\leq \gamma \|x-y\|_{\infty}.
\end{align}
This means that $\hat{\mathbf T}_{\pi}$ is a contraction, which implies that it has a fixed point.

We then show the limit of the fixed points of $\hat{\mathbf T}_{\pi}$ is the fixed point of ${\mathbf T}_{\pi}$
Note that $\mathbf T_{\pi}V_1=V_1$ and $\hat{\mathbf T}_{\pi}V_2=V_2$, hence 
\begin{align}
    &\|V_1-V_2\|_{\infty}\nn\\
    &=\|\mathbf T_{\pi}V_1-\hat{\mathbf T}_{\pi}V_2\|_{\infty}\nn\\
    &\leq \|\mathbf T_{\pi}V_1-\mathbf T_{\pi}V_2 \|_{\infty} +\|\mathbf T_{\pi}V_2-\hat{\mathbf T}_{\pi}V_2\|_{\infty}\nn\\
    &=\max_{s} \bigg|\mE_{\pi}\bigg[\gamma\left(1-R\right)\sum_{s'}p^A_{s,s'}V_1\left(s'\right)+\gamma R \max_{s'}V_1\left(s'\right)\nn\\
    &\quad-\gamma\left(1-R\right)\sum_{s'}p^A_{s,s'}V_2\left(s'\right)-\gamma R \max_{s'}V_2\left(s'\right)\bigg]\bigg|\nn\\
    &\quad + \max_{s}\left|\mE_{\pi}\left[\gamma R \left(\max_{s'}V_2\left(s'\right)-\text{LSE}(V_2)\right)\right]\right|\nn\\
    &{\leq} \max_{s}\mE_{\pi}\left[\left |\gamma \left(1-R\right) \sum_{s'} p^A_{s,s'}\left(V_1\left(s'\right)-V_2\left(s'\right)\right) \right|+ \left|\gamma R \left(\max_{s'}V_1\left(s'\right)-\max_{s'}V_2\left(s'\right)\right) \right| \right]\nn\\
    &\quad + \max_{s}\left|\mE_{\pi}\left[\gamma R \left(\max_{s'}V_2\left(s'\right)-\text{LSE}(V_2)\right)\right]\right|\nn\\
    &\overset{(a)}{\leq} \max_{s}\gamma |V_1\left(s\right)-V_2\left(s\right)|+ \left|\mE_{\pi}\left[\gamma R \left(\max_{s'}V_2\left(s'\right)-\text{LSE}(V_2)\right)\right]\right|\nn\\
    &\leq \gamma \|V_1-V_2\|_{\infty}+\gamma R \frac{\log |\mcs|}{\varrho},
\end{align}
where $(a)$ is from $|V_1(s')-V_2(s')|\leq \max_{s} |V_1(s)-V_2(s)|=\|V_1-V_2\|_{\infty}$ and $|\max_{s'}V_1(s')-\max_{s'}V_2(s')|\leq \|V_1-V_2\|_{\infty}$, and the last inequality is from $\text{LSE}(V)-\max V =\frac{\log (\sum_s e^{\varrho V(s)})-\log e^{\varrho \max V}}{\varrho}=\frac{1}{\varrho} \log \frac{\sum_s e^{\varrho V(s)}}{ e^{\varrho \max V}}=\frac{1}{\varrho} \log {\sum_s e^{\varrho V(s)-\varrho \max V}}\leq \frac{\log |\mcs|}{\varrho}$.
Hence this completes the proof.

\section{Proof of Theorem \ref{thm:TDC}: Finite-Time Analysis of Robust TDC with Linear Function Approximation}
In this section we develop the finite-time analysis of the robust TDC algorithm. In the following proofs, $\|v\|$ denotes the $l_2$ norm if $v$ is a vector; and $\|A\|$ denotes the operator norm if $A$ is a matrix. 

For the convenience of proof, we add a projection step to the algorithm, i.e., we let 
\begin{align}
     \theta_{t+1}&  \leftarrow  \mathbf \Pi_K \left(\theta_{t}+\alpha\left(\delta_t(\theta_t)\phi_t-\gamma\bigg( (1-R)\phi_{t+1}+ R \sum\limits_{s\in\mcs} \left(\frac{e^{\varrho V_{\theta}(s)}\phi_s}{\sum_{j\in\mcs}e^{\varrho V_{\theta}(j)}}\right)\bigg)\phi_t^\top\omega_t \right)\right),\nn\\
	 \omega_{t+1}&\leftarrow \mathbf\Pi_K\left(\omega_{t}+\beta(\delta_{t}(\theta_{t})-\phi_{t}^\top  \omega_{t})\phi_{t} \right),
\end{align}
for some constant $K$. 
We note that recently there are several works \citep{srikant2019,xu2021sample,kaledin2020finite} on finite-time analysis of RL algorithms that do not need the projection. However, a direct generalization of their approach does not necessarily work in our case. Specifically, the problem in \citep{srikant2019} is for one time scale \textit{linear} stochastic approximation.
and doesn't need to consider the effect of the $\omega_t$ introduced, also their work highly depends on the bound of the update functions of $\theta_t$ (see inequality (18) in \citep{srikant2019}). 
The parameter $\theta_t$ in \citep{srikant2019} is bounded using itself at a previous timestep by taking advantage of the fact that the update of $\theta$ is linear. However, in our problem, the update is not linear in $\theta$, and our update rule is two time-scale. The approach in \citep{kaledin2020finite} transforms the original two time-scale updates into two asymptotically independent updates via a linear mapping, which is however challenging for our non-linear updates. 
Some other work, e.g., \citep{xu2021sample}, gets around this issue by imposing additional assumptions on the function class. Specifically, it is assumed that $V_\theta$ (non-linear function approximation) is bounded for all $\theta$. For the linear function approximation setting considered in this paper, this assumption is equivalent to the assumption of a finite $\theta$, which is guaranteed by the projection step in this paper.

\subsection{Lipschitz Smoothness}
In this section, we first show that $\nabla J(\theta)$ is Lipschitz. We begin with an important lemma. 
\begin{Lemma}\label{lemma:delta}
For any $(s,a,s')\in\mcs\times\mca\times\mcs$, both $\delta_{s,a,s'}(\theta)$ and $\nabla \delta_{s,a,s'}(\theta)$ are bounded and Lipschitz, i.e., for any $\theta$ and $\theta'$,
\begin{align}
    |\delta_{s,a,s'}(\theta)|&\leq c_{\max}+\gamma R (K+\frac{\log|\mcs|}{\varrho})+(1+\gamma) K\triangleq C_{\delta},\\
    \|\delta_{s,a,s'}(\theta)-\delta_{s,a,s'}(\theta') \|&\leq (1+\gamma) \|\theta-\theta'\|\triangleq L_{\delta}\|\theta-\theta' \|,\\
    \|\nabla \delta_{s,a,s'}(\theta)- \nabla \delta_{s,a,s'}(\theta')\| &\leq 2\gamma R\varrho\|\theta-\theta'\|\triangleq L'_{\delta}\|\theta-\theta' \|.
\end{align}
\end{Lemma}
\begin{proof}

\textbf{1. $\delta$ is bounded:}

Recall that 
\begin{align}
    \delta_{s,a,s'}(\theta)=c(s,a)+\gamma (1-R) V_{\theta}(s')+\gamma R \frac{\log (\sum_{j\in \mcs}e^{\varrho \theta^\top\phi_j})}{\varrho}-V_{\theta}(s).
\end{align}
First we have that 
\begin{align}
    |\delta_{s,a,s'}(\theta)|&\leq c_{\max}+\gamma (1-R) K+\gamma R \frac{\log |\mcs|e^{K\varrho}}{\varrho}+\gamma R K +K \nn\\
    &=c_{\max}+\gamma R (K+\frac{\log|\mcs|}{\varrho})+(1+\gamma) K. 
\end{align}
\textbf{2. $\delta$ is Lipschitz:}

The Lipschitz smoothness of $\delta_{s,a,s'}$ can be showed by finding the bound of $\nabla \delta_{s,a,s'}$. We first recall that 
\begin{align}
    \nabla  \delta_{s,a,s'}(\theta)=\gamma (1-R) \phi_{s'}+\gamma R \frac{\sum_i e^{\varrho \theta^\top \phi_i}\phi_i}{\sum_j e^{\varrho \theta^\top \phi_j}}-\phi_s.
\end{align}
Hence
\begin{align}
    \|\nabla  \delta_{s,a,s'}(\theta)\|\leq \gamma (1-R)+1+\gamma R =1+\gamma.
\end{align}

\textbf{3. $\nabla\delta$ is Lipschitz:}

Finally we need to verify the Lipschitz smoothness of $\nabla \delta_{s,a,s'}(\theta)$, which can be implied from the bound of $\nabla^2 \delta_{s,a,s'}(\theta)$. First we have that 
\begin{align}\label{eq:nndelta}
    \nabla^2 \delta_{s,a,s'}(\theta)=\gamma R\varrho\frac{\sum_{i,j}e^{\varrho \theta^\top (\phi_i+\phi_j)}\phi_i^\top\phi_i-\sum_{i,j}e^{\varrho \theta^\top (\phi_i+\phi_j)}\phi_i^\top\phi_j}{(\sum_j e^{\varrho \theta^\top\phi_j})^2}\leq 2\gamma R\varrho.
\end{align}
\end{proof}
With this lemma, we then show that $\nabla J(\theta)$ is Lipschitz as follows. 
\begin{Lemma}\label{lemma:Jlsmooth}
For any $\theta$ and $\theta'$, we have that
\begin{align}
    \|\nabla J(\theta)-\nabla J(\theta')\|\leq 2\left(\frac{L_{\delta}^2}{\lambda}+\frac{C_{\delta}L'_{\delta}}{\lambda}\right)\|\theta-\theta'\|\triangleq L_J \|\theta-\theta'\|.
\end{align}
\end{Lemma}
\begin{proof}
From Lemma \ref{lemma:delta}, we have that 
\begin{align}
    \left\|\mathbb{E}_{\mu_\pi}[(\nabla \delta_{S,A,S'}(\theta)) \phi_S]\right\|\leq L_{\delta}
\end{align}
and 
\begin{align}
     \left\|\mathbb{E}_{\mu_\pi}[(\nabla \delta_{S,A,S'}(\theta)) \phi_S]-\mathbb{E}_{\mu_\pi}[(\nabla \delta_{S,A,S'}(\theta')) \phi_S]\right\|\leq L'_{\delta}\|\theta-\theta'\|. 
\end{align}
Also it is easy to see that 
\begin{align}
    \|C^{-1}\mathbb{E}_{\mu_\pi}[\delta_{S,A,S'}(\theta)\phi_S]\|\leq \frac{1}{\lambda}C_{\delta},
\end{align}
and 
\begin{align}
    \|C^{-1}\mathbb{E}_{\mu_\pi}[\delta_{S,A,S'}(\theta)\phi_S]-C^{-1}\mathbb{E}_{\mu_\pi}[\delta_{S,A,S'}(\theta')\phi_S]\|\leq \frac{1}{\lambda} L_{\delta}\|\theta-\theta'\|. 
\end{align}
Thus this implies that 
\begin{align}
    \|\nabla J(\theta)-\nabla J(\theta')\|\leq 2\left(\frac{L_{\delta}^2}{\lambda}+\frac{C_{\delta}L'_{\delta}}{\lambda}\right)\|\theta-\theta'\|,
\end{align}
and hence completes the proof.
\end{proof}
\subsection{Tracking Error}
In this section, we study the bound of the tracking error, which is defined as $z_t=\omega_t-\omega(\theta_t)$.
First we can rewrite the fast time-scale update in Algorithm \ref{alg:1} as follows:
\begin{align}
    z_{t+1}&=\omega_{t+1}-\omega(\theta_{t+1})\nn\\
    &=\omega_t+\beta (\delta_t(\theta_t)-\phi_t^\top \omega_t)\phi_t-\omega(\theta_{t+1})\nn\\
    &=z_t+\omega(\theta_t)+\beta (\delta_t(\theta_t)-\phi_t^\top \omega_t)\phi_t-\omega(\theta_{t+1})\nn\\
    &=z_t+\omega(\theta_t)+\beta (\delta_t(\theta_t)-\phi_t^\top (z_t+\omega(\theta_t)))\phi_t-\omega(\theta_{t+1})\nn\\
    &=z_t+\omega(\theta_t)+\beta \delta_t(\theta_t)\phi_t-\beta \phi_t^\top z_t\phi_t-\beta \phi_t^\top \omega(\theta_t)\phi_t-\omega(\theta_{t+1})\nn\\
    &=z_t-\beta \phi_t\phi_t^\top z_t+ \beta ( \delta_t(\theta_t)\phi_t-\phi_t\phi_t^\top \omega(\theta_t))+\omega(\theta_t)-\omega(\theta_{t+1}).
\end{align}
Thus taking the norm of both sides implies that 
\begin{align}\label{eq:1}
    \|z_{t+1}\|^2&\overset{(a)}{\leq} \|z_t\|^2+3\beta^2 \|z_t\|^2+3\beta^2 \| \delta_t(\theta_t)\phi_t-\phi_t\phi_t^\top \omega(\theta_t)\|^2+3\| \omega(\theta_t)-\omega(\theta_{t+1})\|^2\nn\\
    &\quad+2\langle z_t, -\beta \phi_t\phi_t^\top z_t\rangle +2 \langle z_t, \beta ( \delta_t(\theta_t)\phi_t-\phi_t\phi_t^\top \omega(\theta_t))\rangle + 2\langle z_t, \omega(\theta_t)-\omega(\theta_{t+1})\rangle\nn\\
    &=\|z_t\|^2-2\beta z_t^\top C z_t +3\beta^2 \|z_t\|^2+3\beta^2 \| \delta_t(\theta_t)\phi_t-\phi_t\phi_t^\top \omega(\theta_t)\|^2+3\| \omega(\theta_t)-\omega(\theta_{t+1})\|^2\nn\\
    &\quad+2\beta\langle z_t, (C-\phi_t\phi_t^\top) z_t\rangle +2 \langle z_t, \beta ( \delta_t(\theta_t)\phi_t-\phi_t\phi_t^\top \omega(\theta_t))\rangle + 2\langle z_t, \omega(\theta_t)-\omega(\theta_{t+1})\rangle\nn\\
    &\overset{(b)}{\leq} (1+3\beta^2-2\beta\lambda) \|z_t\|^2+\beta^2 C_1+2\beta\langle z_t, (C-\phi_t\phi_t^\top) z_t\rangle+ 2\langle z_t, \omega(\theta_t)-\omega(\theta_{t+1})\rangle\nn\\
     &\quad +2 \langle z_t, \beta ( \delta_t(\theta_t)\phi_t-\phi_t\phi_t^\top \omega(\theta_t))\rangle,
\end{align}
where$(a)$ is from $\|x+y+z\|^2\leq 3\|x\|^2+3\|y\|^2+3\|z\|^2$ for any $x,y,z \in \mathbb R^N$, $(b)$ is from $z_t^\top C z_t\geq \lambda \|z_t\|^2$, and $C_1=3\left(C_{\delta}+\frac{C_{\delta}}{\lambda}\right)^2+3\left(C_{\delta}+(1+2R\varrho K)\frac{C_{\delta}}{\lambda}\right)^2$ is the upper bound of $3 \| \delta_t(\theta_t)\phi_t-\phi_t\phi_t^\top \omega(\theta_t)\|^2+\frac{3}{\beta^2}\| \omega(\theta_t)-\omega(\theta_{t+1})\|^2$.

Taking expectation on both sides and applying recursively \eqref{eq:1}, we obtain that 
\begin{align}\label{eq:trackingrecur}
    \mE[\|z_{t+1}\|^2] &\leq q^{t+1} \|z_0\|^2+2\sum^t_{j=0} q^{t-j} \beta \mE[f(z_j,O_j)]+ 2\sum^t_{j=0} q^{t-j} \beta \mE[g(z_j,\theta_j,O_j)]\nn\\
    &\quad+2\sum^t_{j=0} q^{t-j} \langle z_j, \omega(\theta_j)-\omega(\theta_{j+1})\rangle+ \beta^2C_1\sum^t_{j=0} q^{t-j},
\end{align}
where
\begin{align}\label{eq:fgdef}
    q&\triangleq 1+3\beta^2-2\beta\lambda,\nn\\
    f(z_j,O_j)&\triangleq \langle z_j, (C-\phi_j\phi_j^\top) z_j\rangle,\nn\\
    g(z_j,\theta_j,O_j) &\triangleq\langle z_j,   \delta_j(\theta_j)\phi_j-\phi_j\phi_j^\top \omega(\theta_j)\rangle.
\end{align}

To simplify notations, let
\begin{align}
    \theta_{t+1}\leftarrow \theta_t+\alpha G_t(\theta_t,\omega_t),\\
    \omega_{t+1}\leftarrow \omega_t+\beta H_t(\theta_t,\omega_t),
\end{align}
where $G_t(\theta,\omega)= \delta_t(\theta)\phi_t-\gamma\bigg( (1-R)\phi_{t+1}+ R \frac{\sum_i e^{\varrho \theta^\top \phi_i}\phi_i}{\sum_j e^{\varrho \theta^\top\phi_j}}\bigg)\phi_t^\top\omega   $, and 
$H_t(\theta,\omega)= (\delta_{t}(\theta_{t})-\phi_{t}^\top  \omega_{t})\phi_{t}$.

We have
\begin{align}\label{eq:Gbound}
    \|G_t(\theta,\omega)\|\leq C_{\delta}+K\gamma\triangleq C_G.
\end{align}
The upper bound of $H_t(\theta,\omega)$ is straightforward:
\begin{align}\label{eq:Hbound}
    \|H_t(\theta,\omega)\|\leq C_{\delta}+K\triangleq C_H.
\end{align}
With these two bounds we can then find the upper bound of the update of tracking error:
\begin{align}\label{eq:zupdatebound}
    \|z_{t+1}-z_t\|&\leq\|H_t(\theta_t,\omega_t)\|+\|\omega(\theta_{t+1})-\om(\theta_t)\|\nn\\
    &\overset{(a)}{\leq} \beta C_H+\alpha\frac{C_{\delta}}{\lambda}\|G_t(\theta_t,\om_t)\|\nn\\
    &\leq \beta C_H+\alpha\frac{C_{\delta}C_G}{\lambda},
\end{align}
where $(a)$ is from the Lipschitz of $\omega(\theta)$: $\|\om(\theta_{t+1})-\om(\theta_t)\|\leq \frac{L_{\delta}}{\lambda}\|\theta_{t+1}-\theta_t\|\leq \frac{\alpha L_{\delta}}{\lambda}\|G_t(\theta_t,\om_t)\|$.
Then for the Lipschitz smoothness of function $g$ in \eqref{eq:fgdef}, it is straightforward to see that 
\begin{align}\label{eq:glip}
    &|g(\theta,z,O_t)-g(\theta',z',O_t)|\nn\\
    &=\langle z,   \delta_j(\theta)\phi_j-\phi_j\phi_j^\top \omega(\theta)\rangle-\langle z',   \delta_j(\theta')\phi_j-\phi_j\phi_j^\top \omega(\theta')\rangle\nn\\
    &=\langle z,   \delta_j(\theta)\phi_j-\phi_j\phi_j^\top \omega(\theta)\rangle-\langle z,   \delta_j(\theta')\phi_j-\phi_j\phi_j^\top \omega(\theta')\rangle\nn\\
    &\quad+ \langle z,   \delta_j(\theta')\phi_j-\phi_j\phi_j^\top \omega(\theta')\rangle-\langle z',   \delta_j(\theta')\phi_j-\phi_j\phi_j^\top \omega(\theta')\rangle\nn\\
    &\leq K_zL_{\delta}\left(1+\frac{1}{\lambda}\right)\|\theta-\theta'\|+C_{\delta}\left(1+\frac{1}{\lambda}\right)\|z-z'\|,
\end{align}
where $K_z\triangleq K+\frac{C_{\delta}}{\lambda}$ being a rough bound on the track error. 
Also it can be shown that 
\begin{align}\label{eq:flip}
    |f(z,O_t)-f(z',O_t)|&=\langle z, (C-\phi_t\phi_t^\top) z \rangle -\langle z', (C-\phi_t\phi_t^\top) z' \rangle\nn\\
    &=\langle z, (C-\phi_t\phi_t^\top) z \rangle-\langle z, (C-\phi_t\phi_t^\top) z' \rangle\nn\\
    &\quad +\langle z, (C-\phi_t\phi_t^\top) z' \rangle-\langle z', (C-\phi_t\phi_t^\top) z' \rangle\nn\\
    &\leq 4K_z\|z-z'\|.
\end{align}
It is easy to see that 
\begin{align}\label{eq:G2lip}
    \|G_i(\theta,\om_1)-G_i(\theta,\om_2)\|\leq (\gamma+2\gamma R\varrho K)\|\om_1-\om_2\|. 
\end{align}

With these bounds and Lipschitz constants, the following two lemmas can be proved using the similar method of decoupling the Markovian noise in \citep{wang2020finite,bhandari2018finite,zou2019finite}.
\begin{Lemma}
Define $\tau_{\beta}=\min\left\{k: m\rho^k\leq \beta \right\}$. If $t<\tb$, then \begin{align}
    \mE[f(z_t,O_t)]\leq 4K_z^2;
\end{align} and if $t\geq \tb$, then \begin{align}
    \mE[f(z_t,O_t)]\leq m_f \beta+ m'_f \tb\beta,
\end{align} where $m_f=8K_z^2$ and $m'_f=8K_z\left( C_H+\frac{C_GC_{\delta}}{\lambda}\right)$.
\end{Lemma}
A similar result on $\mE[g(\theta_t,z_t,O_t)]$ can also be implied:   
\begin{Lemma}
If $t<\tb$, then \begin{align}
    \mE[g(\theta_t,z_t,O_t)]\leq 2K_z\left(1+\frac{1}{\lambda}\right)C_{\delta};
\end{align} and if $t\geq \tb$, then \begin{align}
    \mE[g(\theta_t,z_t,O_t)]\leq m_g \beta+ m'_g \tb\beta,
\end{align} where $m_g=4K_z\left( 1+\frac{1}{\lambda} \right)C_{\delta}$ and $m'_g=4K_zL_{\delta}C_G\left( 1+\frac{1}{\lambda} \right)+C_{\delta}\left( 1+\frac{1}{\lambda} \right)\left( C_H+\frac{C_GC_{\delta}}{\lambda} \right)$.
\end{Lemma}
One more lemma is needed to bound the tracking error.
\begin{Lemma}\label{lemma:h}
 Define $h(\theta,z,O_t)=\left\langle z, -\nabla \omega(\theta)\left(G_{t}(\theta,\omega(\theta))+\frac{\nabla J(\theta)}{2}\right)\right\rangle$, then if $t<\tb$, \begin{align}
    \mE[h(\theta_t,z_t,O_t)]\leq K_zC_h;
\end{align} and if $t\geq \tb$,  \begin{align}
    \mE[h(\theta_t,z_t,O_t)]\leq m_h \beta+ m'_h \tb\beta,
\end{align} where $m_h=2K_zC_h$ and $m_h'=C_h \left(  C_H+\frac{C_{\delta}C_G}{\lambda}\right)+K_zL_hC_G$.
\end{Lemma}
\begin{proof}
First we show the Lipschitz smoothness of $h$ as follows. For any $\theta, \theta', z$ and $z'$, we have that 
\begin{align}
    &h(\theta,z,O_t)-h(\theta',z',O_t)\nn\\
    &=\left\langle z, -\nabla \omega(\theta)\left(G_{t}(\theta,\omega(\theta))+\frac{\nabla J(\theta)}{2}\right)\right\rangle-\left\langle z', -\nabla \omega(\theta')\left(G_{t}(\theta',\omega(\theta'))+\frac{\nabla J(\theta')}{2}\right)\right\rangle\nn\\
    &=\left\langle z, -\nabla \omega(\theta)\left(G_{t}(\theta,\omega(\theta))+\frac{\nabla J(\theta)}{2}\right)\right\rangle-\left\langle z', -\nabla \omega(\theta)\left(G_{t}(\theta,\omega(\theta))+\frac{\nabla J(\theta)}{2}\right)\right\rangle\nn\\
    &\quad+\left\langle z', -\nabla \omega(\theta)\left(G_{t}(\theta,\omega(\theta))+\frac{\nabla J(\theta)}{2}\right)\right\rangle-\left\langle z', -\nabla \omega(\theta')\left(G_{t}(\theta',\omega(\theta'))+\frac{\nabla J(\theta')}{2}\right)\right\rangle.
\end{align}
We note that 
\begin{align}
    &\left\|-\nabla \omega(\theta)\left(G_{t}(\theta,\omega(\theta))+\frac{\nabla J(\theta)}{2}\right)\right\|\nn\\
    &\leq \frac{L_{\delta}}{\lambda} \left( C_{\delta}+\gamma (1-R)+2\varrho K\gamma R \frac{C_{\delta}}{\lambda} +\frac{2L_{\delta}C_{\delta}}{\lambda} \right)\triangleq C_h,
\end{align}
and 
\begin{align}
    &\left\|-\nabla \omega(\theta)\left(G_{t}(\theta,\omega(\theta))+\frac{\nabla J(\theta)}{2}\right)+\nabla \omega(\theta')\left(G_{t}(\theta',\omega(\theta'))+\frac{\nabla J(\theta')}{2}\right)\right\|\nn\\
    &\leq \left(\frac{L'_{\delta}}{L_{\delta}}C_h+\frac{L_{\delta}L_{G^*}}{\lambda} +\frac{L_{\delta}L_J}{2\lambda}\right)\|\theta-\theta'\|\triangleq L_h \|\theta-\theta'\|.
\end{align}
Hence we have that 
\begin{align}
     h(\theta,z,O_t)-h(\theta',z',O_t)\leq C_h\|z-z'\|+K_zL_h \|\theta-\theta'\|.
\end{align}
We have shown before in \eqref{eq:zupdatebound} that \begin{align}
    \|z_{t+1}-z_t\|\leq \beta C_H+\alpha\frac{C_{\delta}C_G}{\lambda}. 
\end{align}
Hence, we have that  
\begin{align}
    |h(\theta_t,z_t,O_t)-h(\theta_{t-\tau},z_{t-\tau},O_t)| \leq C_h \left( \beta C_H+\alpha\frac{C_{\delta}C_G}{\lambda}\right) \tau+K_zL_hC_G\tau \alpha.
\end{align}
Define an independent random variable $\hat{O}=(\hat{S},\hat{A},\hat{S'})\sim \mu_\pi \times \mathsf P(\cdot|\hat{S},\hat{A})$, then we have 
\begin{align}
    \mE_{\hat{O}}[h(\theta,z,\hat{O})]=0
\end{align}
for any $\theta$ and $z$. Thus  by uniform ergodicity, we have that
\begin{align}
    \mE[h(\theta_{t-\tau},z_{t-\tau},O_t)]\leq \mE[h(\theta_{t-\tau},z_{t-\tau},O_t)]-\mE_{\hat{O}}[h(\theta_t,z_t,\hat{O})] \leq 2K_zC_h m\rho^{\tau}. 
\end{align}
Then if $t\leq \tb$, we have the straightforward bound 
\begin{align}
    \mE[h(\theta_t,z_t,O_t)]\leq K_zC_h;
\end{align}
and if  $t> \tb$, we have that 
\begin{align}
     \mE[h(\theta_t,z_t,O_t)]&\leq \mE[h(\theta_{t-\tb},z_{t-\tb},O_t)]+C_h \left( \beta C_H+\alpha\frac{C_{\delta}C_G}{\lambda}\right) \tb+K_zL_hC_G\tb \alpha\nn\\
     &\leq 2K_zC_h m\rho^{\tb}+C_h \left( \beta C_H+\alpha\frac{C_{\delta}C_G}{\lambda}\right) \tb+K_zL_hC_G\tb \alpha\nn\\
     &\triangleq m_h \beta+ m_h' \tb\beta,
\end{align}
where $m_h=2K_zC_h$ and $m_h'=C_h \left(  C_H+\frac{C_{\delta}C_G}{\lambda}\right)+K_zL_hC_G$. This completes the proof.
\end{proof}
Now we bound the tracking error in \eqref{eq:trackingrecur}. We first rewrite it as
\begin{align} 
    \mE[\|z_{t+1}\|^2] &\leq q^{t+1} \|z_0\|^2+\underbrace{2\sum^t_{j=0} q^{t-j} \beta \mE[f(z_j,O_j)]}_{A_t}+ \underbrace{2\sum^t_{j=0} q^{t-j} \beta \mE[g(z_j,\theta_j,O_j)]}_{B_t}\nn\\
    &\quad+\underbrace{2\sum^t_{j=0} q^{t-j} \langle z_j, \omega(\theta_j)-\omega(\theta_{j+1})\rangle}_{C_t}+ \beta^2C_1\sum^t_{j=0} q^{t-j}.
\end{align}
The second term $A_t$ can be bounded as follows:
\begin{align}
    A_t&=2\sum^t_{j=0} q^{t-j} \beta \mE[f(z_j,O_j)]\nn\\
    &=2\sum^{\tb-1}_{j=0} q^{t-j} \beta \mE[f(z_j,O_j)]+2\sum^t_{j=\tb} q^{t-j} \beta \mE[f(z_j,O_j)]\nn\\
    &\leq 8\sum^{\tb-1}_{j=0} q^{t-j} K_z \beta +2\sum^t_{j=\tb} q^{t-j} \beta (m_f\beta+m'_f \tb\beta)\nn\\
    &\leq 16K_z\beta \frac{q^{t+1-\tb}}{1-q}+2\beta (m_f\beta+m'_f \tb\beta)\frac{1-q^{t-\tb+1}}{1-q}.
\end{align}
Similarly, we have that 
\begin{align}
     B_t \leq 4K_z\beta \left(1+\frac{1}{\lambda}\right)C_{\delta} \frac{q^{t+1-\tb}}{1-q}+2\beta (m_g\beta+m'_g \tb\beta)\frac{1-q^{t-\tb+1}}{1-q}.
\end{align}
For $C_t$, we first note that 
\begin{align}
    &\mathbb{E}\left[\left\langle z_i, \omega\left(\theta_i\right)-\omega\left(\theta_{i+1}\right)\right\rangle\right]\nn\\
    &\overset{(a)}{=}\mathbb{E}\left[\left\langle z_i, \nabla \omega\left(\theta_i\right)\left(\theta_i-\theta_{i+1}\right)+R_2\right\rangle\right]\nn\\
    &=\mathbb{E}\left[\left\langle z_i, -\alpha\nabla \omega\left(\theta_i\right)G_i\left(\theta_i,\omega_i\right)+R_2\right\rangle\right]\nn\\
    &=\mathbb{E}\bigg[\bigg\langle z_i, -\alpha\nabla \omega\left(\theta_i\right)\left(G_i\left(\theta_i,\omega_i\right)-G_i\left(\theta_i,\omega\left(\theta_i\right)\right)+G_i\left(\theta_i,\omega\left(\theta_i\right)\right)+\frac{\nabla J\left(\theta_i\right)}{2}-\frac{\nabla J\left(\theta_i\right)}{2}\right)\nn\\
    &\quad+R_2\bigg\rangle\bigg]\nn\\
    &= \underbrace{\mathbb{E}\left[\left\langle z_i, -\alpha\nabla \omega\left(\theta_i\right)\left(G_i\left(\theta_i,\omega\left(\theta_i\right)\right)+\frac{\nabla J\left(\theta_i\right)}{2}\right)\right\rangle\right]}_{\left(b\right)}\nn\\
    &\quad+\underbrace{\mathbb{E}\left[\left\langle z_i, -\alpha\nabla \omega\left(\theta_i\right)\left(G_i\left(\theta_i,\omega_i\right)-G_i\left(\theta_i,\omega\left(\theta_i\right)\right)-\frac{\nabla J\left(\theta_i\right)}{2}\right)+R_2\right\rangle\right]}_{\left(c\right)},
\end{align}
where $(a)$ follows from the Taylor expansion, and $R_2$ is the remaining term with norm $\|R_2\|=\mathcal{O}(\alpha^2)$. 
Term $\left(b\right)$ can be bounded using Lemma \ref{lemma:h}, where
\begin{align}
    \mathbb{E}\left[\left\langle z_i, -\alpha\nabla \omega\left(\theta_i\right)\left(G_i\left(\theta_i,\omega\left(\theta_i\right)\right)+\frac{\nabla J\left(\theta_i\right)}{2}\right)\right\rangle\right]=\alpha \mathbb{E}\left[h\left(\theta_i,z_i,O_i\right)\right].
\end{align}
Term $\left(c\right)$ can be bounded as follows.
\begin{align}
   & \left\langle z_i, -\alpha\nabla \omega\left(\theta_i\right)\left(G_i\left(\theta_i,\omega_i\right)-G_i\left(\theta_i,\omega\left(\theta_i\right)\right)-\frac{\nabla J\left(\theta_i\right)}{2}\right)+R_2\right\rangle\nn\\
   &\overset{(d)}{\leq} \frac{\lambda \beta}{8}\|z_i\|^2+\frac{2}{\lambda\beta}\left\|\alpha\nabla \omega\left(\theta_i\right)\left(G_i\left(\theta_i,\omega_i\right)-G_i\left(\theta_i,\omega\left(\theta_i\right)\right)-\frac{\nabla J\left(\theta_i\right)}{2}\right)+R_2 \right\|^2\nn\\
   &\leq \frac{\lambda \beta}{8}\|z_i\|^2\nn\\
   &\quad+\frac{6}{\lambda\beta}\left(\left\|\alpha\nabla \omega\left(\theta_i\right)\left(G_i\left(\theta_i,\omega_i\right)-G_i\left(\theta_i,\omega\left(\theta_i\right)\right)\right)\right\|^2+\left\|\alpha\nabla \omega\left(\theta_i\right)\frac{\nabla J\left(\theta_i\right)}{2} \right\|^2+\|R_2\|^2\right)\nn\\
   &\leq \frac{\lambda \beta}{8}\|z_i\|^2+\frac{6\alpha^2}{\lambda\beta} \frac{L^2_{\delta}}{\lambda^2}(\gamma+2\gamma R\varrho K)^2\|z_i\|^2+\frac{3\alpha^2}{2\lambda\beta}\frac{L^2_{\delta}}{\lambda^2}\|\nabla J(\theta_i) \|^2+\frac{6}{\lambda\beta}\|R_2\|^2.
\end{align}
where $(d)$ is from $\langle x,y\rangle \leq \frac{\lambda\beta}{8}\|x\|^2+\frac{2}{\lambda\beta}\|y\|^2$ for any $x,y \in \mathbb{R}^N$ and the fact that  $\|G_i(\theta,\omega_1)-G_i(\theta,\omega_2)\|\leq (\gamma+2\gamma \varrho RK)\|\omega_1-\omega_2\|$ for any $\|\theta\|\leq R$ and $\omega_1, \omega_2$, which is from \eqref{eq:G2lip}  . 

Finally the term $C_t$ can be bounded as follows. 
\begin{align}
    C_t&=2\sum^t_{j=0} q^{t-j} \langle z_j, \omega(\theta_j)-\omega(\theta_{j+1})\rangle\nn\\
    &=2\sum^t_{j=0} q^{t-j} \alpha \mE[h(\theta_j,z_j,O_j)]\nn\\
    &\quad+2\sum^t_{j=0} q^{t-j} \left(\frac{\lambda \beta}{8}\|z_i\|^2+\frac{6\alpha^2}{\lambda\beta} \frac{L^2_{\delta}}{\lambda^2}(\gamma+2\gamma R\varrho K)^2\|z_i\|^2+\frac{3\alpha^2}{2\lambda\beta}\frac{L^2_{\delta}}{\lambda^2}\|\nabla J(\theta_i) \|^2+\frac{6}{\lambda\beta}\|R_2\|^2 \right)\nn\\
    &\triangleq 2\sum^t_{j=0} q^{t-j} \alpha \mE[h(\theta_j,z_j,O_j)]+M_t,
\end{align}
where $M_t=2\sum^t_{j=0} q^{t-j} \left(\frac{\lambda \beta}{8}\|z_i\|^2+\frac{6\alpha^2}{\lambda\beta} \frac{L^2_{\delta}}{\lambda^2}(\gamma+2\gamma R\varrho K)^2\|z_i\|^2+\frac{3\alpha^2}{2\lambda\beta}\frac{L^2_{\delta}}{\lambda^2}\|\nabla J(\theta_i) \|^2+\frac{6}{\lambda\beta}\|R_2\|^2 \right)$. From Lemma \ref{lemma:h}, we have that 
\begin{align}
    &2\sum^t_{j=0} q^{t-j} \alpha \mE[h(\theta_j,z_j,O_j)]\nn\\
    &\leq 2\alpha\left(\sum^{\tb-1}_{j=0} q^{t-j}\mE[h(\theta_j,z_j,O_j)]+ \sum^t_{j=\tb}q^{t-j}\mE[h(\theta_j,z_j,O_j)] \right)\nn\\
    &\leq 4K_zC_h\alpha\sum^{\tb-1}_{j=0}q^{t-j}+ 2\alpha(m_h\beta+m'_h\tb\beta)\sum^t_{j=\tb}q^{t-j}\nn\\
    &= 4K_zC_h \alpha\frac{q^{t+1-\tb}}{1-q}+2\alpha(m_h\beta+m'_h\tb\beta)\frac{1-q^{t-\tb+1}}{1-q},
\end{align}
and this implies that 
\begin{align}
    C_t\leq 4K_zC_h\alpha \frac{q^{t+1-\tb}}{1-q}+2\alpha(m_h\beta+m'_h\tb\beta)\frac{1-q^{t-\tb+1}}{1-q}+M_t. 
\end{align}
Now we plug  the bounds on $A_t, B_t$ and $C_t$ in \eqref{eq:trackingrecur}, we have that 
\begin{align}
    &\mE[\|z_{t+1}\|^2]\nn\\
    &\leq q^{t+1}\|z_0\|^2+\beta^2C_1\frac{1-q^{t+1}}{1-q}+\left( 16K_z\beta+4K_zC_{\delta}\beta\left(1+\frac{1}{\lambda}\right)+4K_zC_h\alpha\right)\frac{q^{t+1-\tb}}{1-q}\nn\\
    &\quad+\left(2\beta(m_f\beta+m'_f \tb\beta)+2\beta(m_g\beta+m'_g \tb\beta) +2\alpha(m_h\beta+m'_h \tb\beta) \right)\frac{1-q^{t-\tb+1}}{1-q}+M_t\nn\\
    &\leq  q^{t+1}\|z_0\|^2+\beta^2C_1\frac{1-q^{t+1}}{1-q}+C_z\beta\frac{q^{t+1-\tb}}{1-q}+\beta(m_z\beta+m'_z\tb\beta)\frac{1-q^{t-\tb+1}}{1-q}+M_t,
\end{align}
where $C_z=16K_z+4K_zC_{\delta}\left(1+\frac{1}{\lambda}\right)+4K_zC_h\frac{\alpha}{\beta}$, $m_z=2m_f+2m_g+2\frac{\alpha}{\beta}m_h$ and $m'_z=2m'_f+2m'_g+ \frac{2\alpha}{\beta}m'_h$. Note that $q=1+3\beta^2-2\beta\lambda\triangleq 1-u\beta\leq e^{-u\beta}$, where $u=2\lambda-3\beta$. 
Hence it implies that 
\begin{align}\label{eq:tracking1}
    &\frac{\sum^{T-1}_{t=0}\mE[\|z_t\|^2]}{T}\nn\\
    &\leq \frac{1}{T}\Bigg(\frac{\|z_0\|^2}{1-e^{-u\beta}}+\beta^2C_1\frac{T}{u\beta}+ 4K_z^2\tb\nn\\
    &\quad+\sum^{T-1}_{t=\tb-1}\left( C_z\beta\frac{q^{t+1-\tb}}{u\beta}+\beta(m_z\beta+m'_z\tb\beta)\frac{1-q^{t-\tb+1}}{u\beta}+M_t\right)\Bigg)\nn\\
    &\leq \frac{1}{T}\Bigg(\frac{\|z_0\|^2}{1-e^{-u\beta}}+\beta^2C_1\frac{T}{u\beta}+ 4K_z^2\tb\nn\\
    &\quad+c_z\beta \frac{\sum^{T-1}_{t=0} e^{-ut\beta}}{u\beta}+\beta(m_z\beta+m'_z\tb\beta)\frac{T}{u\beta}+\sum^{T-1}_{t=0} M_t\Bigg)\nn\\
    &\leq \frac{1}{T}\Bigg(\frac{\|z_0\|^2}{1-e^{-u\beta}}+\beta^2C_1\frac{T}{u\beta}+ 4K_z^2\tb +c_z\beta \frac{1}{(u\beta)(1-e^{-u\beta})}+\beta(m_z\beta+m'_z\tb\beta)\frac{T}{u\beta}\nn\\
    &\quad+\sum^{T-1}_{t=0} M_t\Bigg)\nn\\
    &=\frac{1}{T}\Bigg(\frac{\|z_0\|^2}{1-e^{-u\beta}}+\beta C_1\frac{T}{u}+ 4K_z^2\tb+\frac{c_z}{u(1-e^{-u\beta})}+(m_z\beta+m'_z\tb\beta)\frac{T}{u}+\sum^{T-1}_{t=0} M_t\Bigg)\nn\\
    &\leq \frac{\|z_0\|^2}{T(1-e^{-u\beta})}+\beta\frac{ C_1 }{u}+ 4K_z^2\frac{\tb}{T}+\frac{c_z}{u(1-e^{-u\beta})T}+(m_z\beta+m'_z\tb\beta)\frac{1}{u}+\frac{\sum^{T-1}_{t=0} M_t}{T}\nn\\
    &\triangleq Q_T+\frac{\sum^{T-1}_{t=0} M_t}{T}\nn\\
    &=\mathcal{O}\left( \frac{1}{T\beta}+\beta\tb+\frac{\tb}{T}+\frac{\sum^{T-1}_{t=0} M_t}{T}\right),
\end{align}
where $Q_T=\frac{\|z_0\|^2}{T(1-e^{-u\beta})}+\beta\frac{ C_1 }{u}+ 4K_z^2\frac{\tb}{T}+\frac{c_z}{u(1-e^{-u\beta})T}+(m_z\beta+m'_z\tb\beta)\frac{1}{u}$.

We then compute $\sum^{T-1}_{t=0} M_t$. Recall that $M_t=2\sum^t_{j=0} q^{t-j} \Big(\frac{\lambda \beta}{8}\|z_i\|^2+\frac{6\alpha^2}{\lambda\beta} \frac{L^2_{\delta}}{\lambda^2}(\gamma+2\gamma R\varrho K)^2\|z_i\|^2+\frac{3\alpha^2}{2\lambda\beta}\frac{L^2_{\delta}}{\lambda^2}\|\nabla J(\theta_i) \|^2+\frac{6}{\lambda\beta}\|R_2\|^2 \Big)$. From double sum trick, i.e., $\sum^{T-1}_{t=0}\sum^t_{i=0} e^{-u(t-i)\beta}x_i \leq \frac{1}{1-e^{-u\beta}}\sum^{T-1}_{t=0}x_t$ for any $x_t\geq 0$, we have that 
\begin{align}
    \sum^{T-1}_{t=0} M_t&\leq \frac{2}{1-e^{-u\beta}}\left(\frac{\lambda\beta}{8}+ \frac{6\alpha^2}{\lambda\beta} \frac{L^2_{\delta}}{\lambda^2}(\gamma+2\gamma R\varrho K)^2\right)\sum^{T-1}_{t=0}\mE[\|z_t\|^2]\nn\\
    &\quad+\frac{2}{1-e^{-u\beta}} \frac{3\alpha^2}{2\lambda\beta}\frac{L^2_{\delta}}{\lambda^2} \sum^{T-1}_{t=0}\mE[\|\nabla J(\theta_t)\|^2]+\frac{6}{\lambda\beta}\frac{2}{1-e^{-u\beta}}\|R_2\|^2T.
\end{align}
Note that $1-e^{-u\beta}=\mathcal{O}(\beta)$, thus we can choose $\alpha$ and $\beta$ such that $\frac{2}{1-e^{-u\beta}}\left(\frac{\lambda\beta}{8}+ \frac{6\alpha^2}{\lambda\beta} \frac{L^2_{\delta}}{\lambda^2}(\gamma+2\gamma R\varrho K)^2\right)\leq \frac{1}{2}$, then by plugging $\sum^{T-1}_{t=0} M_t$ in \eqref{eq:tracking1} we have that 
\begin{align}
    \frac{1}{2}\frac{\sum^{T-1}_{t=0}\mE[\|z_t\|^2]}{T}\leq Q_T+\frac{2}{1-e^{-u\beta}} \frac{3\alpha^2}{2\lambda\beta}\frac{L^2_{\delta}}{\lambda^2} \frac{\sum^{T-1}_{t=0}\mE[\|\nabla J(\theta_t)\|^2]}{T}+\frac{6}{\lambda\beta}\frac{2}{1-e^{-u\beta}}\|R_2\|^2,
\end{align}
and this implies that 
\begin{align}\label{eq:trackingbound}
    \frac{\sum^{T-1}_{t=0}\mE[\|z_t\|^2]}{T}&\leq 2Q_T+\frac{2}{1-e^{-u\beta}} \frac{3\alpha^2}{\lambda\beta}\frac{L^2_{\delta}}{\lambda^2} \frac{\sum^{T-1}_{t=0}\mE[\|\nabla J(\theta_t)\|^2]}{T}+\frac{6}{\lambda\beta}\frac{4}{1-e^{-u\beta}}\|R_2\|^2\nn\\
    &=\mathcal{O}\left(  \frac{1}{T\beta}+\beta\tb+ \frac{\alpha^2}{\beta^2}\frac{\sum^{T-1}_{t=0}\mE[\|\nabla J(\theta_t)\|^2]}{T}\right),
\end{align}
which completes the development of error bound on the tracking error. 
\subsection{Finite-Time Error Bound}\label{sec:finitebound}
Now with the tracking error in \eqref{eq:trackingbound}, we derive the finite-time error of the robust TDC. From Lemma \ref{lemma:Jlsmooth} and Taylor expansion, we have that
\begin{align}\label{eq:Jlsmooth}
    J(\theta_{t+1}) &\leq J(\theta_t) +\left\langle \nabla J(\theta_t), \theta_{t+1}-\theta_t\right\rangle  + \frac{L_J}{2} \| \theta_{t+1}-\theta_t\|^2\nn\\
    &=J(\theta_t) +\alpha \left\langle \nabla J(\theta_t),G_t(\theta_t,\omega_t) \right\rangle  + \frac{L_J}{2} \alpha^2||G_t(\theta_t,\omega_t)||^2\nn\\
    &=J(\theta_t)-\alpha\left\langle \nabla J(\theta_t),-G_t(\theta_t, \omega_t)-\frac{\nabla J(\theta_t)}{2}+G_t(\theta_t, \omega(\theta_t))-G_t(\theta_t, \omega(\theta_t)) \right\rangle \nn\\
    &\quad-\frac{\alpha}{2}||\nabla J(\theta_t)||^2+\frac{L_J}{2} \alpha^2||G_t(\theta_t,\omega_t)||^2\nn\\
    &=J(\theta_t)-\alpha\left\langle \nabla J(\theta_t),-G_t(\theta_t, \omega_t)+G_t(\theta_t, \omega(\theta_t)) \right\rangle\nn\\
    &\quad+\alpha \left\langle \nabla J(\theta_t), \frac{\nabla J(\theta_t)}{2}+G_t(\theta_t, \omega(\theta_t)) \right\rangle-\frac{\alpha}{2}||\nabla J(\theta_t)||^2+\frac{L_J}{2} \alpha^2||G_t(\theta_t,\omega_t)||^2\nn\\
    &\leq J(\theta_t) +\alpha \|\nabla J(\theta_t) \|(\gamma+2\gamma RK\varrho )\|\omega(\theta_t)-\omega_t \|-\frac{\alpha}{2}||\nabla J(\theta_t)||^2 \nn\\
    &\quad+\alpha \left\langle \nabla J(\theta_t), \frac{\nabla J(\theta_t)}{2}+G_t(\theta_t, \omega(\theta_t)) \right\rangle+\frac{L_J}{2} \alpha^2||G_t(\theta_t,\omega_t)||^2.
\end{align}
By taking expectation on both sides and summing up from $0$ to $T-1$, we have that 
 \begin{align}\label{eq:main}
    &\sum^{T-1}_{t=0} \frac{\alpha}{2} \mathbb{E}[\|\nabla J(\theta_t) \|^2] \nn\\
    &\leq J(\theta_0)-J(\theta_{T})+ \alpha(\gamma+2\gamma RK\varrho )\sqrt{\sum^{T-1}_{t=0} \mathbb{E}[\|\nabla J(\theta_t)\|^2]}\sqrt{\sum^{T-1}_{t=0}\mathbb{E}[\|z_t\|^2]} \nn\\
    &\quad+\sum^{T-1}_{t=0}\alpha \mathbb{E}\left[\left\langle \nabla J(\theta_t),\frac{\nabla J(\theta_t)}{2}+G_t(\theta_t, \omega(\theta_t)) \right\rangle\right]+\frac{L_J}{2}\sum^{T-1}_{t=0}\alpha^2 \mathbb{E}[\| G_t(\theta_t,\omega_t)\|^2],
\end{align}
which follows from the Cauchy-Schwartz inequality: $\sum^{T-1}_{t=0} \mathbb{E}[\|\nabla J(\theta_t)\|\|z_t\|]\leq \sum^{T-1}_{t=0}\sqrt{\mathbb{E}[\| \nabla J(\theta_t)\|^2]\mathbb{E}[\|z_t\|^2]}\leq\sqrt{\sum^{T-1}_{t=0}\mathbb{E}[\|\nabla J(\theta_t)\|^2]}\sqrt{\sum^{T-1}_{t=0}\mathbb{E}[\|z_t \|^2]}$.
To bound the Markovian noise term, i.e., $\left\langle \nabla J(\theta), \frac{\nabla J(\theta)}{2}+G_t(\theta, \omega(\theta)) \right\rangle$, we first need some bounds and smoothness conditions. It can be shown that 
\begin{align}
    \|G_t(\theta,\om(\theta))\|&\leq C_{\delta}+\frac{C_{\delta}}{\lambda}(\gamma+2\varrho K\gamma R)\triangleq C_{G*},\\
    \|G_t(\theta,\om(\theta))-G_t(\theta',\om(\theta'))\|&\leq \left(L_{\delta}+\frac{L_{\delta}}{\lambda}(\gamma+2\gamma R\varrho K)+\frac{C_{\delta}}{\lambda}L'_{\delta} \right)\|\theta-\theta'\|\triangleq L_{G*}\|\theta-\theta'\|.
\end{align}
\begin{Lemma}
Define $\zeta(\theta, O_t)\triangleq\left\langle \nabla J(\theta), \frac{\nabla J(\theta)}{2}+G_t(\theta, \omega(\theta)) \right\rangle$, and let $\tau_{\alpha}\triangleq \min \left\{k : m\rho^k \leq \alpha \right\}$. 
If $t < \tau_{\alpha}$, then
\begin{align}
    \mE[\zeta(\theta_t,O_t)]\leq\frac{C_{\delta}L_{\delta}}{\lambda}\left(\frac{C_{\delta}L_{\delta}}{2\lambda}+C_{G*}\right) \triangleq C_{\zeta};
\end{align}
and if $t \geq \tau_{\alpha}$, then 
\begin{align}
    \mE[\zeta(\theta_t,O_t)]\leq m_{\zeta}\alpha+m'_{\zeta}\tau_{\alpha}\alpha,
\end{align}
where $m_{\zeta}=2C_{\zeta}$ and $m'_{\zeta}=C_G\left(\frac{L_JC_{\delta}L_{\delta}}{\lambda}+\frac{C_{\delta}L_{\delta}L_{G*}}{\lambda}+L_JC_{G*} \right)$.
\end{Lemma}
Next we plug the tracking error \eqref{eq:trackingbound} in \eqref{eq:main}.
\begin{align} 
    &\sum^{T-1}_{t=0} \frac{\alpha}{2} \mathbb{E}[\|\nabla J(\theta_t) \|^2] \nn\\
    &\leq J(\theta_0)-J(\theta_{T})+ \alpha(\gamma+2\gamma RK\varrho )\sqrt{\sum^{T-1}_{t=0} \mathbb{E}[\|\nabla J(\theta_t)\|^2]}\sqrt{2TQ_T+2\sum^{T-1}_{t=0}M_t} \nn\\
    &\quad+\alpha\tau_{\alpha}C_{\zeta}+\alpha^2(T-\tau_{\alpha})(m_{\zeta}+m'_{\zeta}\tau_{\alpha})+\frac{L_J}{2}\alpha^2 C_G^2T.
\end{align}
Divided both sides by $\frac{\alpha T}{2}$, we have that 
\begin{align}
     &\frac{\sum^{T-1}_{t=0} \mathbb{E}[\|\nabla J(\theta_t) \|^2]}{T} \nn\\
    &\leq \frac{2J(\theta_0)-2J(\theta_{T})}{\alpha T}+ 2(\gamma+2\gamma RK\varrho )\sqrt{\frac{\sum^{T-1}_{t=0} \mathbb{E}[\|\nabla J(\theta_t)\|^2]}{T}}\sqrt{2Q_T+2\frac{\sum^{T-1}_{t=0}M_t}{T}} \nn\\
    &\quad+\frac{2\tau_{\alpha}C_{\zeta}}{T}+2\alpha (m_{\zeta}+m'_{\zeta}\tau_{\alpha})+L_J\alpha C_G^2.
\end{align}
 We know from \eqref{eq:trackingbound} that $ 2\frac{\sum^{T-1}_{t=0}M_t}{T}\leq  \frac{2}{1-e^{-u\beta}} \frac{3\alpha^2}{\lambda\beta}\frac{L^2_{\delta}}{\lambda^2} \frac{\sum^{T-1}_{t=0}\mE[\|\nabla J(\theta_t)\|^2]}{T}+\frac{6}{\lambda\beta}\frac{4}{1-e^{-u\beta}}\|R_2\|^2$, thus 
\begin{align}
     &\frac{\sum^{T-1}_{t=0} \mathbb{E}[\|\nabla J(\theta_t) \|^2]}{T} \nn\\
    &\leq \frac{2J(\theta_0)-2J(\theta_{T})}{\alpha T}+ 2(\gamma+2\gamma RK\varrho )\sqrt{\frac{\sum^{T-1}_{t=0} \mathbb{E}[\|\nabla J(\theta_t)\|^2]}{T}}\nn\\
    &\quad \left(\sqrt{2Q_T+\frac{6}{\lambda\beta}\frac{4}{1-e^{-u\beta}}\|R_2\|^2}+\sqrt{\frac{2}{1-e^{-u\beta}} \frac{3\alpha^2}{\lambda\beta}\frac{L^2_{\delta}}{\lambda^2} \frac{\sum^{T-1}_{t=0}\mE[\|\nabla J(\theta_t)\|^2]}{T}}\right) \nn\\
    &\quad+\frac{2\tau_{\alpha}C_{\zeta}}{T}+2\alpha (m_{\zeta}+m'_{\zeta}\tau_{\alpha})+L_J\alpha C_G^2\nn\\
    &=\frac{2J(\theta_0)-2J(\theta_{T})}{\alpha T}+ 2(\gamma+2\gamma RK\varrho )\sqrt{\frac{2}{1-e^{-u\beta}} \frac{3\alpha^2}{\lambda\beta}\frac{L^2_{\delta}}{\lambda^2} } {\frac{\sum^{T-1}_{t=0} \mathbb{E}[\|\nabla J(\theta_t)\|^2]}{T}}\nn\\
    &\quad \left(\sqrt{2Q_T+\frac{6}{\lambda\beta}\frac{4}{1-e^{-u\beta}}\|R_2\|^2} \right)2(\gamma+2\gamma RK\varrho )\sqrt{\frac{\sum^{T-1}_{t=0} \mathbb{E}[\|\nabla J(\theta_t)\|^2]}{T}} \nn\\
    &\quad+\frac{2\tau_{\alpha}C_{\zeta}}{T}+2\alpha (m_{\zeta}+m'_{\zeta}\tau_{\alpha})+L_J\alpha C_G^2\nn\\
    &\triangleq \frac{2J(\theta_0)-2J(\theta_{T})}{\alpha T}+K_1\frac{\sum^{T-1}_{t=0} \mathbb{E}[\|\nabla J(\theta_t)\|^2]}{T}+K_2 \sqrt{\frac{\sum^{T-1}_{t=0} \mathbb{E}[\|\nabla J(\theta_t)\|^2]}{T}}+\frac{2\tau_{\alpha}C_{\zeta}}{T}\nn\\
    &\quad+2\alpha (m_{\zeta}+m'_{\zeta}\tau_{\alpha})+L_J\alpha C_G^2,
\end{align}
where $K_1=2(\gamma+2\gamma RK\varrho )\sqrt{\frac{2}{1-e^{-u\beta}} \frac{3\alpha^2}{\lambda\beta}\frac{L^2_{\delta}}{\lambda^2} }=\mathcal{O}\left( \frac{\alpha}{\beta}\right)$ and $K_2=\left(\sqrt{2Q_T+\frac{6}{\lambda\beta}\frac{4}{1-e^{-u\beta}}\|R_2\|^2} \right)2(\gamma+2\gamma RK\varrho )=\mathcal{O}\left( \sqrt{\frac{\alpha^4}{\beta^2}+\frac{1}{T\beta}+\beta\tb}\right)$. Thus we can choose $\alpha$ and $\beta$ such that $K_1\leq \frac{1}{2}$, then we have that 
\begin{align}
    &\frac{\sum^{T-1}_{t=0} \mathbb{E}[\|\nabla J(\theta_t) \|^2]}{T} \nn\\
    &\leq  \frac{4J(\theta_0)-4J(\theta_{T})}{\alpha T}+2K_2 \sqrt{\frac{\sum^{T-1}_{t=0} \mathbb{E}[\|\nabla J(\theta_t)\|^2]}{T}}+\frac{4\tau_{\alpha}C_{\zeta}}{T}+4\alpha (m_{\zeta}+m'_{\zeta}\tau_{\alpha})+2L_J\alpha C_G^2\nn\\
    &\triangleq U+V\sqrt{\frac{\sum^{T-1}_{t=0}\mathbb{E}[\|\nabla J(\theta_t)\|^2]}{T}},
\end{align}
where $U=\frac{4J(\theta_0)-4J(\theta_{T})}{\alpha T}+\frac{4\tau_{\alpha}C_{\zeta}}{T}+4\alpha (m_{\zeta}+m'_{\zeta}\tau_{\alpha})+2L_J\alpha C_G^2=\mathcal{O}(\alpha\tau_{\alpha}+\frac{1}{\alpha T})$ and $V=2K_2$.
Hence, we have that
\begin{align}\label{eq:tdcbound}
    &\frac{\sum^{T-1}_{t=0}\mathbb{E}[\|\nabla J(\theta_t)\|^2]}{T}\nn\\
    &\leq \left(\frac{V+\sqrt{V^2+4U}}{2}\right)^2\nn\\
    &\overset{(a)}{\leq} V^2+2U\nn\\
    &\leq 16\left({2Q_T+\frac{6}{\lambda\beta}\frac{4}{1-e^{-u\beta}}\|R_2\|^2} \right) (\gamma+2\gamma RK\varrho )^2 +\frac{8J(\theta_0)-8J(\theta_{T})}{\alpha T}+\frac{8\tau_{\alpha}C_{\zeta}}{T}\nn\\
    &\quad+8\alpha (m_{\zeta}+m'_{\zeta}\tau_{\alpha})+4L_J\alpha C_G^2\nn\\
     &=\mathcal{O}\left(\frac{1}{T\alpha}+\alpha\tau_{\alpha}+\frac{1}{T\beta} +\beta\tb\right),
\end{align}
where $Q_T=\frac{\|z_0\|^2}{T(1-e^{-u\beta})}+\beta\frac{ C_1 }{u}+ 4K^2\frac{\tb}{T}+\frac{c_z}{u(1-e^{-u\beta})T}+(m_z\beta+m'_z\tb\beta)\frac{1}{u}$.

\subsection{Constants}
In this section we list all the constants occurred in our proof for the readers' reference.

\begin{align}
     C_{\delta}&= c_{\max}+\gamma R \frac{\log |\mcs|}{\varrho}+(1+\gamma) K,\\
    L_{\delta}&=(1+\gamma),\\
    L'_{\delta}&= 2\gamma R\varrho,\\
    L_J&= 2\left(\frac{L_{\delta}^2}{\lambda}+\frac{C_{\delta}L'_{\delta}}{\lambda}\right),\\
    C_1&=3\left(C_{\delta}+\frac{C_{\delta}}{\lambda}\right)^2+3\left(C_{\delta}+(1+2R\varrho K)\frac{C_{\delta}}{\lambda}\right)^2,\\
    C_G&=C_{\delta}+\gamma K+2\gamma \varrho RK^2,\\
    C_H&= C_{\delta}+K,\\
    K_z&=K+\frac{C_{\delta}}{\lambda},\\
    m_g&=4K_z\left( 1+\frac{1}{\lambda} \right)C_{\delta},\\
    m'_g&=4K_zL_{\delta}C_G\left( 1+\frac{1}{\lambda} \right)+C_{\delta}\left( 1+\frac{1}{\lambda} \right)\left( C_H+\frac{C_GC_{\delta}}{\lambda} \right),\\
    m_f&=8K_z^2,\\
    m'_f&=8K_z\left( C_H+\frac{C_GC_{\delta}}{\lambda}\right),\\
    m_h&=2K_zC_h,\\
    m_h'&=C_h \left(  C_H+\frac{C_{\delta}C_G}{\lambda}\right)+K_zL_hC_G,\\
    C_{G*}&= C_{\delta}+\frac{C_{\delta}}{\lambda}(\gamma+2\varrho K\gamma R),\\
    L_{G*}&=L_{\delta}+\frac{L_{\delta}}{\lambda}(\gamma+2\gamma R\varrho K)+\frac{C_{\delta}}{\lambda}L'_{\delta},\\
    L_h&=\frac{L'_{\delta}}{L_{\delta}}C_h+\frac{L_{\delta}L_{G^*}}{\lambda} +\frac{L_{\delta}L_J}{2\lambda},\\
    C_h&=\frac{L_{\delta}}{\lambda} \left( C_{\delta}+\gamma (1-R)+2\varrho K\gamma R \frac{C_{\delta}}{\lambda} +\frac{2L_{\delta}C_{\delta}}{\lambda} \right),\\
    C_{\zeta}&=\frac{C_{\delta}L_{\delta}}{\lambda}\left(\frac{C_{\delta}L_{\delta}}{2\lambda}+C_{G*}\right),\\
    m_{\zeta}&=2C_{\zeta},\\
    m'_{\zeta}&=C_G\left(\frac{L_JC_{\delta}L_{\delta}}{\lambda}+\frac{C_{\delta}L_{\delta}L_{G*}}{\lambda}+L_JC_{G*} \right)
\end{align}

\section{Experiments}
\textbf{Experiments in Section \ref{sec:6.1}:}

{Frozen Lake Problem.}
We consider a $4\times 4$ Frozen Lake problem. 
%
We set $\gamma=0.96$, $\alpha=0.8$. 



%


{Cart-Pole Problem.}
We set $\gamma=0.95$, $\alpha=0.2$. 

\textbf{Experiments in Section \ref{sec:6.2}:}

{Frozen Lake Problem.} We consider a $4\times 4$ Frozen Lake problem. We set $\alpha=0.1$, $\beta=0.5$ and $\gamma=0.9$. The initialization is $\theta=(1,1,1,1,1)\in \mathbb R^5$ and $\omega=(0,0,0,0,0)$. Each entry of every base function $\phi_s$ is generated uniformly at random between $(0,1)$.

{\textbf{Additional Experiments on the Taxi Problem.}}

We use the same setting as in \Cref{sec:6.1} to demonstrate the robustness of our robust Q-learning algorithm. For the step size and discount factor, we set $\alpha=0.3$ and $\gamma=0.8$. The results are shown in \cref{Fig.TX}, from which the same observation that our robust Q-learning is robust to model uncertainty, and achieves a much higher reward when the mismatch between the training and test MDPs enlarges.
\begin{figure}[htb]
\centering 
\subfigure[ p=0.1, R=0.1]{
\label{Fig.tx1}
\includegraphics[width=0.3\linewidth]{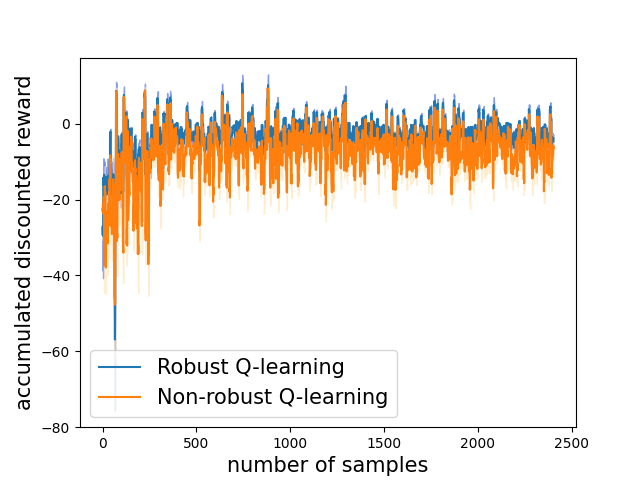}}
\subfigure[ p=0.05, R=0.2]{
\label{Fig.tx2}
\includegraphics[width=0.3\linewidth]{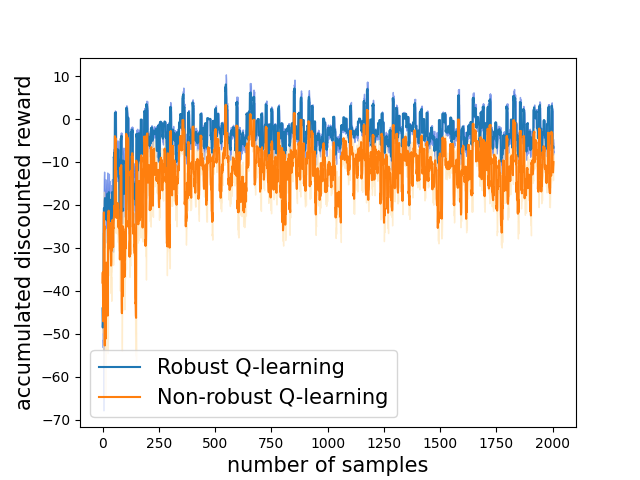}}
\subfigure[ p=0.1, R=0.2]{
\label{Fig.tx3}
\includegraphics[width=0.3\linewidth]{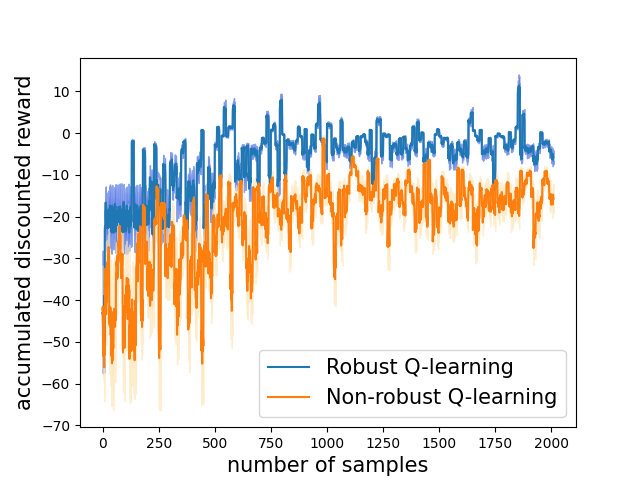}}
\captionsetup{font={normalsize}}
\caption{\textbf{Taxi-v3}: robust Q-learning v.s. non-robust Q-learning.}
\label{Fig.TX}
\end{figure}

\end{document}